\documentclass[twoside,11pt]{article}

\usepackage{todonotes}
\usepackage{bm}

\usepackage{jmlr2e}
\usepackage{amsmath}

\usepackage{amsthm}
\usepackage{verbatim}
\usepackage{algorithm}
\usepackage{algpseudocode}
\usepackage{kbordermatrix} 
\usepackage{dsfont}
\usepackage{amsfonts}
\usepackage{booktabs} 
\usepackage{caption}
\usepackage{subcaption}
\usepackage{natbib}
\usepackage{xcolor}
\usepackage{etoolbox}
\usepackage{lipsum}
\usepackage{thmtools}
\usepackage{booktabs}
\usepackage{multirow}
\usepackage{dirtytalk}
\usepackage{mathtools}
\usepackage{chngcntr}
\usepackage{mathrsfs}
\usepackage[section]{placeins}
\allowdisplaybreaks
\hypersetup{
    colorlinks,
    linkcolor={red!50!black},
    citecolor={blue!50!black},
    urlcolor={blue!80!black}
}

\newtoggle{long-version}
\toggletrue{long-version}

\newcounter{counterCorollary}

\newtheorem{myCorollary}[counterCorollary]{Corollary}
\newtheorem{myTheorem}{Theorem}
\newtheorem{myDefinition}{Definition}
\newtheorem{myLemma}{Lemma}

\newtheorem{fact}{Fact}

\newcommand{\bP}{\mathbb{P}}
\newcommand{\bE}{\mathbb{E}}
\newcommand{\Regret}{\operatorname{Regret}} 
\newcommand{\cRegret}{\operatorname{Regret}} 

\newcommand{\ArmAt}{\hat{a}}
\newcommand{\defined}{\coloneqq}
\newcommand{\constrained}{\stackrel{con}{=}}
\newcommand{\IsEqual}{=}
\newcommand{\LThresholdTS}{u}
\newcommand{\HThresholdTS}{w}
\newcommand{\KLUCBCF}{$\mathrm{kl}$-\textsc{UCB-CF}}
\newcommand{\KLUCB}{$\mathrm{kl}$-\textsc{UCB}}
\newcommand{\TS}{\textsc{TS}}
\newcommand{\TSCF}{\textsc{TS-CF}}
\newcommand{\UCB}{\textsc{UCB1}}
\newcommand{\UCBCF}{\textsc{UCB-CF}}
\newcommand{\cB}{\mathcal{B}}
\newcommand\given[1][]{\:#1\vert\:}
\makeatletter
\newcommand*{\rom}[1]{\expandafter\@slowromancap\romannumeral #1@}
\makeatother
\DeclareMathOperator*{\argmax}{arg\,max}

\renewcommand{\gets}{=}
\algdef{SE}[SUBALG]{Indent}{EndIndent}{}{\algorithmicend\ }%
\algtext*{Indent}
\algtext*{EndIndent}

\def \ind{\mathds{1}}

\ShortHeadings{Corrupt bandits}{Gajane, Urvoy and Kaufmann}
\firstpageno{1}

\begin{document}

\title{Corrupt Bandits for Preserving Local Privacy}

\author{\name Pratik Gajane \email pratik.gajane@inria.fr \\
       \addr Inria Lille Nord-Europe, SequeL team/Orange labs \\
       \name Tanguy Urvoy	\email tanguy.urvoy@orange.com \\
       \addr Orange labs \\
       \name Emilie Kaufmann \email emilie.kaufmann@univ-lille1.fr \\
       \addr CNRS \& Univ. Lille, UMR 9189 (CRIStAL), Inria Lille Nord-Europe, SequeL team
       }
\editor{}

\maketitle

\begin{abstract}
We study a variant of the stochastic multi-armed bandit (MAB) problem in which the rewards are corrupted. In this framework, motivated by privacy preservation in online recommender systems, the goal is to maximize the sum of the (unobserved) rewards, based on the observation of transformation of these rewards through a stochastic corruption process with known parameters. We provide a lower bound on the expected regret of any bandit algorithm in this corrupted setting. We devise a frequentist algorithm, \textsc{KLUCB-CF}, and a Bayesian algorithm, \textsc{TS-CF} and give upper bounds on their regret. We also provide the appropriate corruption parameters to guarantee a desired level of local privacy and analyze how this impacts the regret. Finally, we present some experimental results that confirm our analysis.  
\end{abstract}

\begin{keywords}
Sequential learning, multi-armed bandits, incomplete feedback, local privacy
\end{keywords}

\section{Introduction}
\label{sec_int}
The classical multi-armed bandits (MAB) problem is the formulation of the exploration-exploitation dilemma inherent to reinforcement learning \citep[see][for a survey]{DBLP:journals/ftml/BubeckC12}.
In this setup, a learner has access to a number of available actions, also called ``arms'' in reference to the arm of a slot machine or a one-armed bandit.
They have to repeatedly select (or ``draw'') one of these arms, which yields a reward generated from an unknown reward process, with the aim to maximize the sum of the gathered rewards. After each arm selection, a \emph{feedback} is provided to the learner, that shall influence their arm selection strategy in the next rounds. In the classical MAB problem, the feedback is the observation of the reward itself. However, this assumption does not hold true for some practical scenarios. 

In online advertising, the feedback is typically a user click. However, it is usually given only when it is positive since propagating negative feedback as well is costly in terms of network load, especially on mobile networks.
The reception of a click feedback can be safely interpreted as a positive reward, but the absence of a click (i.e. a timeout) might be a consequence of either a negative reward (since the user did not like the ad) or a bug or a packet loss. 
In adaptive routing, positive feedback means the corresponding path is usable but no feedback could either mean that the corresponding path is unusable or the feedback was dropped due to extraneous issues. 
In the literature, such an \emph{asymmetric feedback} is called \textit{Positive and Unlabeled (PUN)} feedback. 
See \cite{Zhang2008} for a survey.

On-purpose feedback corruption is an effective way to protect the respondent's individual privacy in online recommender systems or survey systems. For instance, \cite{Warner1965} proposed the \emph{randomized response method} as a survey technique to reduce potential bias due to non-response and social desirability when asking questions about sensitive behaviors and beliefs. 
This method asks the respondents to employ randomization, say with a coin flip, the outcome of which is not available to the interviewer. By introducing random noise, the method conceals the individual responses and protects respondent privacy.
This method could also be applied within a recommender system, that would thus receive corrupted version of the user's original feedback about the items presented. Contrary to most previous works which apply privacy at the recommender level, this privacy mechanism, called \textit{local privacy}, can be deployed at the user level.
The challenge for the recommender is then to present good items to the users (in terms of their ``true'' feedback), based only on the received corrupted feedback. Moreover, users may be willing to tune the level of corruption in order to balance between their privacy and the utility of the recommendation they obtain.


The corrupted feedback we consider is a particular type of an incomplete feedback. Therefore, the natural framework to deal with this situation appears to be \emph{Partial Monitoring (PM)} (\cite{conf/colt/PiccolboniS01,BartokFPRS14}), which is a general framework for sequential decision making problems with incomplete feedback.
The partial monitoring problem may be either \emph{trivial} with a minimax regret of $0$, \emph{easy} with a minimax regret $\tilde{\Theta}(\sqrt{T})$ at time $T$, \emph{hard} with a minimax regret $\tilde{\Theta}(T^{2/3})$, or \emph{hopeless} with a linear minimax regret.
The MAB problem with corrupted feedback however does not fit directly in the PM setting as defined in \cite{BartokFPRS14} since it requires additional constraints on the environment. In this work, exploiting the specificity of the corrupted MAB problem, we aim for the best problem-dependent regret, that scales with $\log(T)$.

The article is structured as follows.
In Section~\ref{sec:int}, we formally define the corrupted MAB prob
lem and the relevant parameters. In Section~\ref{sec:lower_bound}, a lower bound on the regret of any corrupt bandit algorithm is given. In Section~\ref{sec:algos}, the algorithms \KLUCBCF \ and \TSCF \ are introduced and we provide upper bounds on their regret. In Section~\ref{sec:privacy}, we describe how corrupted feedback can be used to enforce privacy. The proof sketches for the lower and the upper bounds are given in Section~\ref{sec:sketches}, while the complete proofs are postponed to the appendices.
The penultimate section, Section~\ref{sec:eval}, gives an overview of our experiments on the proposed algorithms.


\section{The Corrupt Bandit Problem}
\label{sec:int}
A (stochastic) corrupt bandit problem $\boldsymbol{\nu}$ is formally characterized by a set of arms $A= \{1, \dots, K\}$ on which are indexed a list of unknown sub-Gaussian
reward distributions $\{\nu_a\}_{a\in A}$,
a list of unknown sub-Gaussian feedback distributions
$\{\varsigma_a\}_{a\in A}$,
and a list of known \textit{mean-corruption functions}
$\{g_a\}_{a\in A}$.


If the learner pulls an arm $ a\in A$ at time $t$, they receive a reward $R_t$ drawn from the distribution $\nu_a$ with mean $\mu^{\boldsymbol{\nu}}_a$ and observe a feedback $F_t$ drawn from the distribution $\varsigma_a$ with mean $\lambda^{\boldsymbol{\nu}}_a$.
We assume that, for each arm, there exists a loose link between the reward and the feedback through a known mean-corruption function (or simply, corruption function) $g_a$ which maps the mean of the reward distribution to the mean of the feedback distribution :
\begin{equation}
\label{eq:link}
\quad g_a(\mu^{\boldsymbol{\nu}}_a)=\lambda^{\boldsymbol{\nu}}_a  , \qquad \forall a\in A
\end{equation}
\noindent
Note that these $g_a$ functions may be completely different from one arm to another.
For Bernoulli distributions, $\mu_a$ and $\lambda_a$ are in $[0,1]$ for all $a \in A$ and we assume all the corruption functions $\{g_a\}_{a \in A}$ to be continuous aleast in this interval.
Let $a_*(\boldsymbol{\nu})\in\arg\max{\mu^{\boldsymbol{\nu}}_a}$ be the optimal arm  in the corrupt bandit model $\boldsymbol{\nu}$\footnote{When the associated model is clear from the context, we drop the symbol $\boldsymbol{\nu}$.}. Without loss of generality, we assume when presenting the results that arm 1 is the optimal arm for the rest of this article, unless otherwise specified.
The objective is to design a strategy, which chooses an arm $\ArmAt_{t}$ to be pulled at time $t$ based only on the previously observed feedback, $F_1,\dots,F_{t-1}$, in order to maximize the expected sum of rewards, or equivalently to minimize the regret: 
$
\Regret_T (\boldsymbol{\nu}) \defined \bE_{\boldsymbol{\nu}}\left[\mu_1 \cdot T - \sum_{t = 1}^T R_t\right] = \sum_{a = 2}^K \Delta_a \cdot \bE_{\boldsymbol{\nu}}[N_a(T)]
$
where $N_a(T) \defined \sum_{t = 1}^T \ind_{(\ArmAt_{t} \IsEqual a)}$ denotes the number of pulls of arm $a$ up to time $T$ and $\Delta_a \defined \mu_1 - \mu_a$ i.e. the gap between the optimal mean reward and the mean reward of arm $a$.



Another way to define the link between the reward and the feedback is to provide a \textit{corruption scheme} operator $\tilde{g}_a$ which maps the reward outcomes into feedback distributions. If the mean is a sufficient statistic of the reward distribution, then the learner can build their own corruption function from the corruption scheme and the two definitions are equivalent.
This equivalence is true for Bernoulli distributions where most of our results apply.

\paragraph{Randomized response.}
Randomized response (\cite{Warner1965}), described in the introduction, can be simulated by a Bernoulli corrupt bandit and the corresponding corruption scheme $\tilde{g}_a$ can be encoded by the matrix:
\begin{equation}
\label{eq:CorruptionMatrix}
  \mathbb{M}_{a} \defined \kbordermatrix{%
      & 0  & 1  \\
    0 & p_{00}(a) & 1 - p_{11}(a) \\
    1 & 1 - p_{00}(a) & p_{11}(a) \\
  }
\end{equation}
where  $ \mathbb{M}_a(y,x) \defined \mathbb{P}( \text{Feedback from arm } a \IsEqual  y \given \text{Reward from arm } a \IsEqual x ).$

The corresponding linear corruption function is 
\begin{equation}
g_a(x) =  1 - p_{00}(a) + [p_{00}(a) + p_{11}(a) - 1]\cdot{} x. \label{eq:RRCorFunction}
\end{equation}

\section{Lower Bound on the Regret for MAB with Corrupted Feedback} 
\label{sec:lower_bound}
Following a definition by \cite{LaiRobbins1985} for the classical MAB, we define a \emph{uniformly efficient} algorithm for the corrupt bandit problem with prescribed corruption functions $\{g_a\}_{a \in A}$ as an algorithm which, for any problem instance $\boldsymbol{\nu}$, has $\Regret_T(\boldsymbol{\nu})=o(T^\alpha)$ for all $\alpha \in ]0,1[$. Theorem~\ref{thm:LB} provides a lower bound on the regret of a uniformly efficient algorithm, in terms of the Kullback-Leibler (KL) divergence between some distributions. We introduce the following notation for the KL-divergence between the Bernoulli distribution of mean $x$ and that of mean $y$: 
\begin{equation}d(x,y):= \mathrm{KL}(\cB(x),\cB(y)) = x \cdot \log\left({x}/{y}\right)+(1-x)\cdot \log\left({(1-x)}/{(1-y)}\right).\label{def:KLdiv} \end{equation}
\begin{myTheorem}\label{thm:LB}
Given continuous corruption functions $\{g_a\}_{a \in A}$, any uniformly efficient algorithm for a Bernoulli corrupt bandit problem satisfies,
$$\liminf_{T\rightarrow \infty}\frac{\cRegret_T}{\log(T)} \geq \sum_{a = 2}^{K} \frac{\Delta_a}{d\left(\lambda_a,g_a(\mu_1)\right)}.$$
\end{myTheorem}
\noindent
The proof of Theorem~\ref{thm:LB} can be found in Section~\ref{sec:Proof_LB_Theorem}.

The lower bound reveals that the divergence between the mean feedback from $a \in A $ and the image of the optimal reward $\mu_1$ with $g_a$ plays a crucial role in distinguishing arm
$a$ from the optimal arm. The shape of the $g_a$ function in the neighborhood of both $a$ and $1$ has a great impact on the information the learner can extract from the received feedback. Particularly, if the $g_a$ function is non-monotonic, as shown in Figure \ref{fig:hardlink}, it might be impossible to distinguish between arm a and the optimal arm. To circumvent this problem, we assume the corruption functions $ \{g_a\}_{a \in A}$ to be strictly monotonic in our algorithms and we denote its corresponding inverse function by $g_a^{-1}$. Such an informative corruption
function is shown in Figure \ref{fig:easylink}. To clarify that the gap between $\lambda_a$ and $\lambda_1$ is not relevant here, we also plot in Figure \ref{fig:easylink}, a corruption function $g_1$ which differs from $g_a$ and causes fortuitously the two arms to have the same mean feedback with different interpretations in terms of mean rewards.

\begin{figure}
\centering
\begin{subfigure}{.5\textwidth}
  \centering
  \includegraphics[width=.9\linewidth]{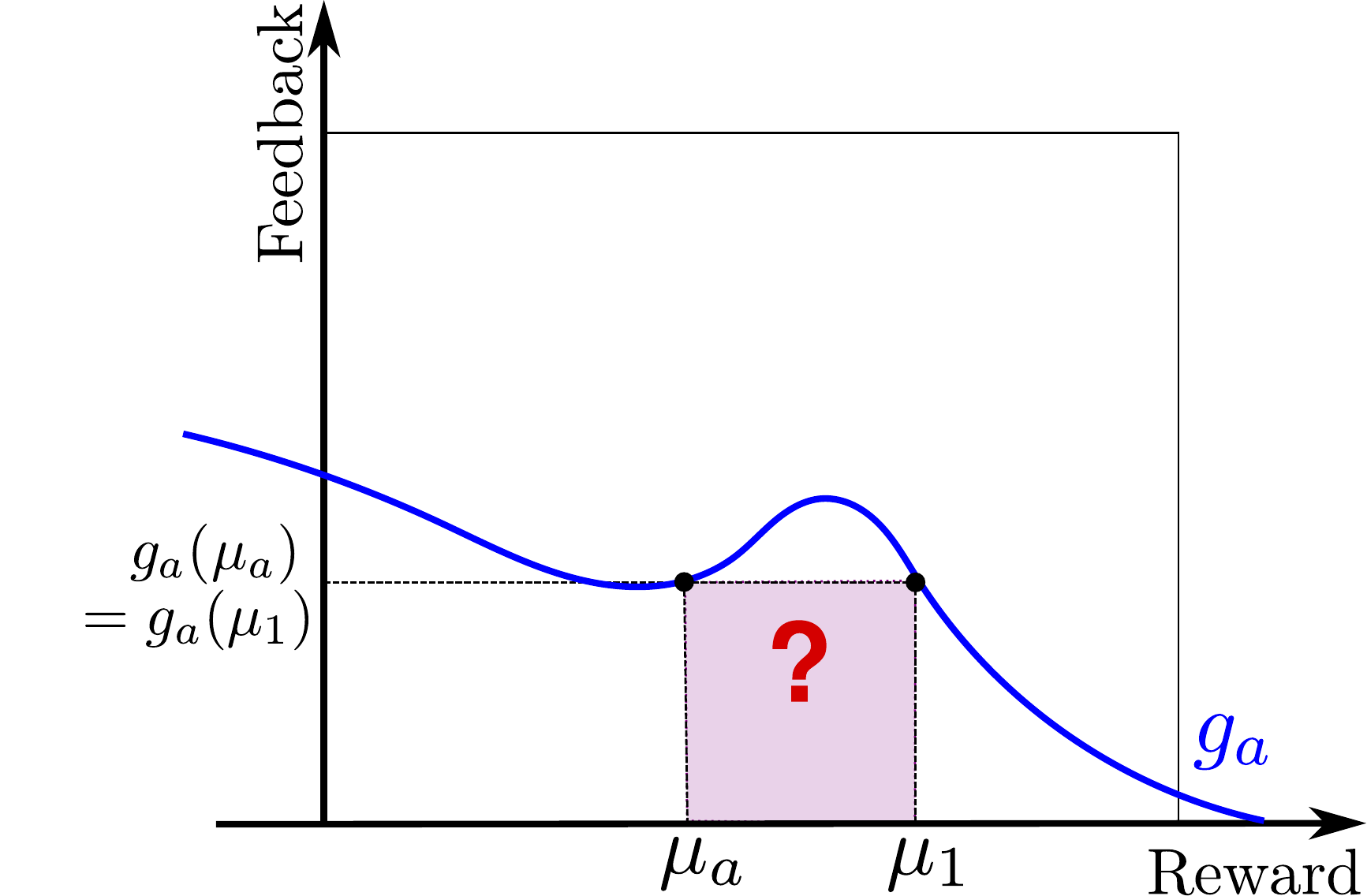}
  \caption{Uninformative $g_a$ function}
  \label{fig:hardlink}
\end{subfigure}%
\begin{subfigure}{.5\textwidth}
  \centering
  \includegraphics[width=.9\linewidth]{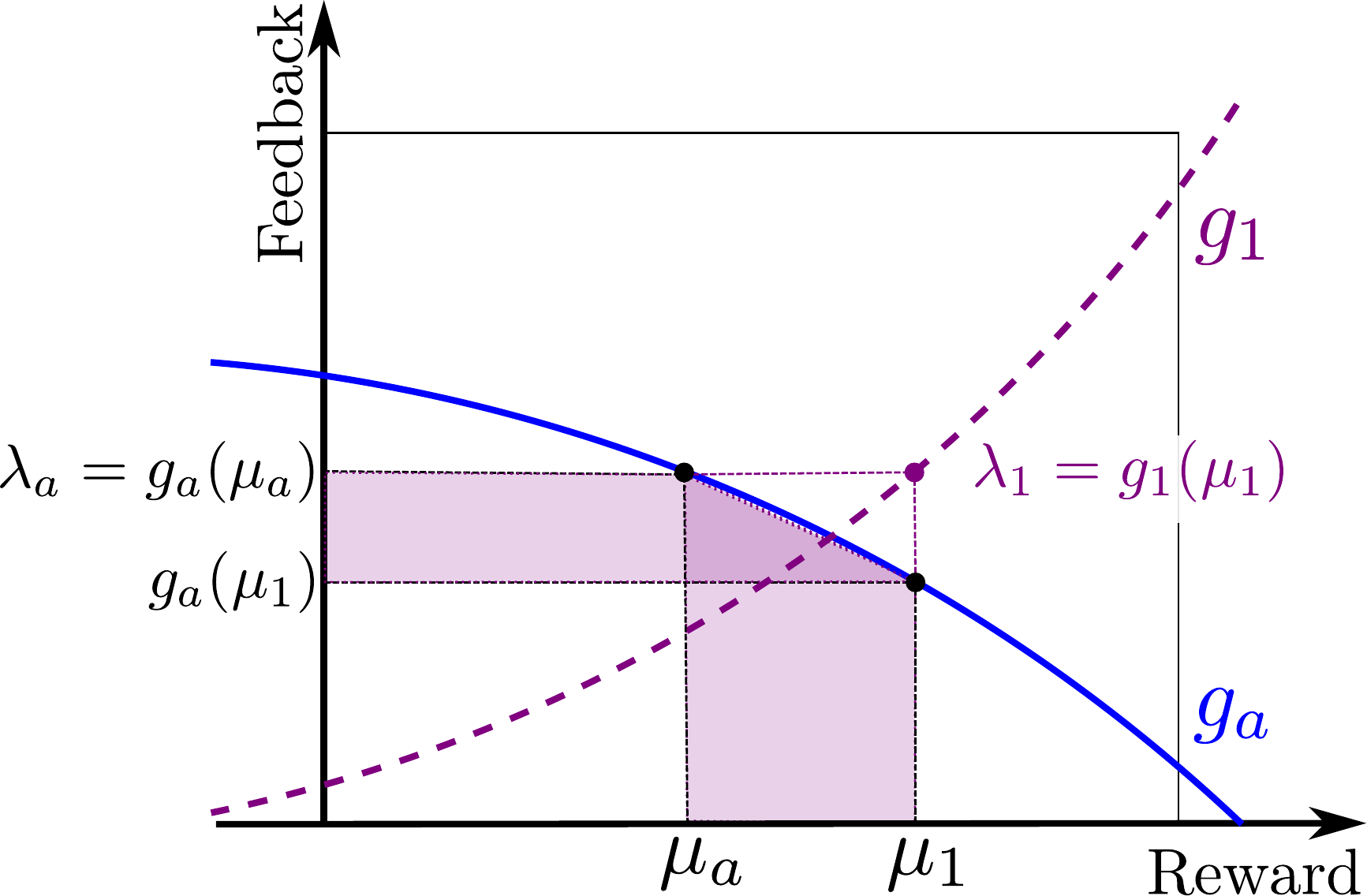}
  \caption{Informative $g_a$ function}
  \label{fig:easylink}
\end{subfigure}
\caption{In Figure \ref{fig:hardlink}, $g_a$ such that $\lambda_a = g_a(\mu_1)$ thereby making it impossible to discern
arm a from the optimal arm given the mean feedback. In Figure \ref{fig:easylink}, a steep monotonic $g_a$ leads the reward gap $\Delta_a = \mu_1 - \mu_a$ into a clear gap between $\lambda$ and $g_a(\mu_1)$. }
\label{fig:link}
\end{figure}

\section{Algorithms for MAB with Corrupted Feedback}
\label{sec:algos}
There are two popular approaches to solve the MAB problem in its many variations: the frequentist approach and the Bayesian approach. In this article, we propose both a frequentist and a Bayesian algorithm for the problem at hand.  

\subsection{\KLUCB \ for MAB with Corrupted Feedback (\KLUCBCF)}
We propose in Algorithm~\ref{algo:KLUCBCF} an adaptation of the \KLUCB \ algorithm of \cite{KLUCBJournal}. $\mathrm{Index}_a(t)$ is an upper-confidence bound on $\mu_a$ built from a confidence interval on $\lambda_a$ based on the KL-divergence \eqref{def:KLdiv}. The quantity $\hat{\lambda}_a(t)$ in the algorithm denotes the empirical mean of the feedback observed from arm $a$ until time $t$: $\hat{\lambda}_a(t) \defined \frac{1}{N_a(t)}\sum_{s=1}^t F_s \ind_{(\ArmAt_s = a)}$. 
\begin{algorithm}[h]
\caption{\KLUCB \ for MAB with corrupted feedback (\KLUCBCF)}
\label{algo:KLUCBCF}
\begin{algorithmic}[1]
\Statex \textbf{Input:} A bandit model with a set of arms $ A \defined \{1, \dots, K\}$ with unknown mean rewards $\mu_1,\dots,\mu_K$ and unknown mean feedbacks $\lambda_1,\dots,\lambda_K$ and monotonic and continuous corruption functions $g_1, \dots, g_k$.
\Statex \textbf{Parameters:} A non-decreasing (exploration) function $f:\mathbb{N} \rightarrow \mathbb{R}$, $d(x,y) \defined \mathrm{KL}(\cB(x),\cB(y))$, Time horizon $T$.
\State \textbf{Initialization:} Pull each arm once.
\For{time $t \gets K, \dots, T-1$}
\State Compute for each arm $a$ in $A$ the quantity
$$\mathrm{Index}_a(t) \defined \max 
\left\{ q:\  N_a(t)\cdot{}d(\hat{\lambda}_a(t), g_a(q)) \leq f(t)\right\}$$
\State Pull arm $\ArmAt_{t+1} \defined \operatornamewithlimits{argmax}\limits_a \hspace{0.25em}{\operatorname{Index}_a(t)}$ and observe the feedback $F_{t+1}$.
\EndFor
\end{algorithmic}
\end{algorithm}

Theorem~\ref{UB_KLUCB_Theorem}  gives an upper bound on the regret of \KLUCBCF, showing that it matches the lower bound given in Theorem~\ref{thm:LB}. A more explicit finite-time bound is proved in Appendix~\ref{sec:Proof_UB_KLUCBCF_Theorem}.
\begin{myTheorem}
\label{UB_KLUCB_Theorem} 
\KLUCBCF \ using  $f(t) \defined \log(t)+3\log(\log(t))$ on a $K$-armed 
Bernoulli corrupt bandit with strictly monotonic and continuous corruption functions $\{g_a\}_{a \in A}$ satisfies at time $T$,
\begin{align*}
\Regret_T &\leq \sum_{a = 2}^{K}\frac{\Delta_a \log (T) }{d\left(\lambda_a,g_a(\mu_1)\right)}  + O(\sqrt{\log(T)}).
\end{align*}
\end{myTheorem}

The \UCB \ algorithm (\cite{Auer2002a}) can also be updated to \UCBCF \ to deal with the corrupted feedback by modifying the index to
\begin{equation*}
    \mathrm{Index}_a(t) \defined
    \begin{cases}
      g_a^{-1}\Big ( \hat{\lambda}_a(t) +  \sqrt{\frac{f(t)}{2 N_a(t)}} \Big) & \text{if increasing}\ g_a\\
      g_a^{-1}\Big ( \hat{\lambda}_a(t) -  \sqrt{\frac{f(t)}{2 N_a(t)}} \Big) & \text{if decreasing}\ g_a
    \end{cases}
  \end{equation*}
\begin{myCorollary}
\label{UB_UCB_Theorem} With $f(t) \defined \log(t)+3\log(\log(t))$,
the regret of \textsc{UCB-CF} at time $T$ on a $K$-armed Bernoulli corrupt bandit with strictly monotonic and continuous corruption functions $\{g_a\}_{a \in A}$ is in $O\Big( \sum_{a = 2}^{K}\frac{ \Delta_a \log(T)}{ (\lambda_a - g_a(\mu_1))^2 } \Big)$.
\end{myCorollary}
\noindent
The proof of this corollary follows the proof of Theorem \ref{UB_KLUCB_Theorem}, using the quadratic divergence $2(x - y)^2$ in place of $d(x,y)$ through Pinsker's inequality.  \textsc{UCB-CF} is only order optimal with respect to the bound of Theorem~\ref{thm:LB}, but its index is simpler to compute.
\subsection{Thompson Sampling for MAB with Corrupted Feedback (\TSCF)}
\TSCF \ maintains a Beta posterior distribution on the mean feedback of each arm. At time $t+1$, for each arm $a$, it draws a sample $\theta_a(t)$ from the posterior distribution on $\lambda_a$ and pulls the arm which maximizes $g_a^{-1}(\theta_a(t))$. This mechanism ensures that at each time, the probability that arm $a$ is played is the posterior probability of this arm to be optimal, as in classical Thompson Sampling (\TS) (\cite{thompson1933}).
\begin{algorithm}[h]
\caption{Thompson sampling for MAB with corrupted feedback (\TSCF)}
\label{algo:TSCF}
    \begin{algorithmic}[1]
    \Statex \textbf{Input:} A bandit model with a set of arms $ A \defined \{1, \dots, K\}$ arms with unknown reward means $\mu_1,\dots,\mu_K$ and unknown feedback means $\lambda_1,\dots,\lambda_K$ and monotonic and continuous corruption functions $g_1, \dots, g_K$.
    \Statex \textbf{Parameters:} Time horizon $T$.
    \State \textbf{Initialization:} For each arm $a$ in $A$, set $\operatorname{success}_a \gets 0$ and $\operatorname{fail}_a \gets 0$
    \For{$t \gets 0, \dots,T-1$}
    \State For each arm $a$ in $A$, sample 
    $\theta_a(t)$ from $\mathrm{Beta}(\operatorname{success}_a+1, \operatorname{fail}_a + 1)
    $.
    \State Pull arm $\ArmAt_{t+1} \defined \arg\max\limits_a g_a^{-1}(\theta_a(t))$ and observe the feedback $F_{t+1}$.
    \If{$F_{t+1} \IsEqual 1$} 
    \State	$\operatorname{success}_{\ArmAt_{t+1}} \gets \operatorname{success}_{\ArmAt_{t+1}} + 1$
    \Else
    \State	$\operatorname{fail}_{\ArmAt_{t+1}} \gets \operatorname{fail}_{\ArmAt_{t+1}} + 1$
    \EndIf
    \EndFor
    \end{algorithmic}
\end{algorithm}

\begin{myTheorem}\label{thm:TSCF}
\label{UB_TSCF_Theorem}
When \TSCF \ is run on a $K$-armed Bernoulli corrupt bandit with strictly monotonic and continuous corruption functions $\{g_a\}_{a \in A}$, for all $\psi >0$, there exists a constant $C_\psi \defined C(\psi,\{\mu_a\}_{a \in A},\{g_a\}_{a \in A})$ such that at time $T$,
\[\cRegret_T \leq (1+\psi)\sum_{a = 2}^K\frac{\Delta_a\log(T)}{d(\lambda_a,g_a(\mu_1))} + C_\psi.\]
\end{myTheorem}
\noindent
This theorem also yields the asymptotic optimality of \TSCF \ with respect to the lower bound given in Theorem~\ref{thm:LB}. 
We give a sketch of its proof in Section~\ref{Sksec:Proof_UB_TSCF_Theorem}. 

We can use the above algorithms on a MAB problem with randomized response. The following corollary bounds their regret.
\begin{myCorollary}
\label{UB_corollary}
The regret of \KLUCBCF \ and \TSCF \ for a $K$-armed Bernoulli MAB problem with randomized response using corruption matrices $\{\mathbb{M}\}_{a \in A}$ at time $T$ is   
$$ \sum_{a = 2}^{K}\frac{2\log(T)}{ \Delta_a (p_{00}(a) + p_{11}(a) -1)^2 } + O(\sqrt{\log{(T)}} ).$$
\end{myCorollary}
\noindent
This corollary follows from Theorem~\ref{UB_KLUCB_Theorem} and Theorem~\ref{UB_TSCF_Theorem} together with Pinsker's inequality: $d(x,y) > 2(x-y)^2$.
The term $(p_{00}(a) + p_{11}(a) -1)$ is the slope of the corruption function for arm $a$ as can be seen from Eq. (\ref{eq:RRCorFunction}).
\section{Corrupted Feedback to Preserve Local Differential Privacy}
\label{sec:privacy}
\textit{Differential privacy} (DP), introduced by \cite{Dwork06calibratingnoise}, is one of the usual approaches for the privacy concerns. \cite{Dwork:2014:AFD:2693052.2693053} present a comprehensive overview. 
\cite{DBLP:journals/jmlr/JainKT12,DBLP:conf/nips/ThakurtaS13,DBLP:conf/uai/MishraT15,tossou:aaai2016} have observed the importance of privacy to MAB applications. Recently, the notion of differential privacy has been extended to \textit{local differential privacy} by \cite{Duchi:2014:PAL:2700084.2666468} in which data remains private even from the learner.
\begin{myDefinition}
\label{def:diff_privacy}
(Locally differentially private mechanism) Any randomized mechanism $\mathcal{M}$ is $\epsilon$-locally differentially private for $\epsilon \geq 0$ if for all  $d_1, d_2 \in Domain(\mathcal{M})$ and for all $S \subset Range(\mathcal{M})$,
$$
\bP[\mathcal{M}(d_1) \in S] \leq e^\epsilon \cdot \bP [\mathcal{M}(d_2) \in S]
$$
\end{myDefinition}
\noindent
In both global and local contexts, differential privacy is achieved by the addition of noise. The main difference between global and local differential privacy is whether privacy is to be maintained from the algorithm or the (possibly unintended) recipient of the output of the algorithm. In global differential privacy, noise is added by the algorithm so the output does not reveal private information about the input. In local differential privacy, noise is added to the input of the algorithm so that privacy is maintained even from the algorithm. To the best of our knowledge, hitherto all the previous work combining differential privacy and bandits has used global differential privacy, either within a stochastic (\cite{DBLP:conf/uai/MishraT15}, \cite{tossou:aaai2016}) or an adversarial (\cite{DBLP:conf/nips/ThakurtaS13}, \cite{tossou:aaai2017a}) bandit problem. 

In this article, we consider local differential privacy. To understand the motivation for local differential privacy, let us consider these settings in the context of Internet advertising, which is one of the major applications of bandit algorithms. An advertising system receives, as input, feedback from the users which may reveal private information about them. The advertising system employs a suitable bandit algorithm and  selects the ads for the users tailored to the feedback given by them. These selected ads are then given to the advertisers as the output \footnote{This description does not express our belief of how real-life Internet advertising systems work. We use it for the purpose of illustration only.}. While using global differential privacy, privacy is maintained from the advertisers by ensuring that the output of the bandit algorithms does not reveal information about the input (i.e. user information). Typically, advertising systems are established by leading social networks, web browsers and other popular websites. \citet{DBLP:conf/icdm/Korolova10}, \cite{kosinski2013private} show that it is possible to accurately predict a range of highly sensitive personal attributes including age, sexual orientation, relationship status, political and religious affiliation, presence or absence of a particular interest, as well as exact birthday using the the feedback available to the advertising systems. Such possible breach of privacy necessitates us to protect personal user information not only from the advertisers but also from the advertising systems. Local differential privacy is able to achieve this goal unlike global differential privacy.

Recently, \cite{DBLP:conf/edbt/0009WH16} addressed a similar scenario in data collection. They used randomized response to perturb sensitive information before being collected by an untrusted server so as to limit the server's ability to learn the sensitive information with confidence. We too shall use the corruption process as a mechanism to provide local differential privacy.             

\begin{myDefinition}
($\epsilon$-locally differentially private bandit feedback corruption scheme) A bandit feedback corruption scheme $\tilde{g}$ is $\epsilon$-locally differentially private for $\epsilon \geq 0$ if for all reward sequences $R_{t1}, \dots, R_{t2}$ and $R'_{t1} \dots, R'_{t2}$, and for all $\mathcal{S} \subset Range(\tilde{g})$ 
$$ \bP [\tilde{g}(R_{t1}, \dots, R_{t2}) \in \mathcal{S} ] \leq e^\epsilon \cdot \bP [\tilde{g}(R'_{t1}, \dots, R'_{t2}) \in \mathcal{S} ].$$

\end{myDefinition}
In the case where corruption is done by randomized response, 
local differential privacy requires that 
$$\max_{ 1 \leq a \leq K}{ \Big( \frac{p_{00}(a)}{1 - p_{11}(a)}, \frac{p_{11}(a)}{1 - p_{00}(a)} \Big)} \leq e^{\epsilon} $$ 
By ensuring the appropriate values for the parameters of randomized response, users can send differentially private feedback to the learner. The learner can then employ \KLUCBCF \ or \TSCF \ to learn from such feedback. From Corollary~\ref{UB_corollary}, we can see that to achieve lower regret, $p_{00}(a) + p_{11}$(a) is to be maximized for all $a \in A$. Using Result 1 from \citet{DBLP:conf/edbt/0009WH16}, we can state that, in order to achieve $\epsilon$-local differential privacy while maximizing $p_{00}(a) + p_{11}(a)$,
\begin{equation}
\label{eq:DPmatrix}
\mathbb{M}_{a}=\kbordermatrix{%
      & 0  & 1  \\
    0 & \frac{e^{\epsilon}}{ 1 + e^\epsilon} & \frac{1}{ 1 + e^\epsilon } \\
    1 & \frac{1}{ 1 + e^\epsilon } & \frac{e^{\epsilon}}{ 1 + e^\epsilon} \\
  }.
\end{equation}
As it turns out, this is equivalent to the \textit{staircase} mechanism for local privacy given in \citet[Eq. (15)]{JMLR:v17:15-135} for binary rewards and feedbacks. Moreover, this is the optimal local differential privacy mechanism for \textit{low privacy regime} \cite[Theorem 14]{JMLR:v17:15-135}. In low privacy regime, the noise added to the data is small and the aim of the privacy mechanism is to send as much information about data as allowed, but no more \citep{NIPS2014_5392}. This is in alignment with our dual goal of using privacy with bandit algorithms: learn from the data while respecting the privacy as much as possible. The trade-off between utility and privacy is controlled by $\epsilon$. At one extreme, for $\epsilon=0$, feedbacks are independent of rewards and learning about rewards from feedbacks is not possible. This can be verified by substituting $\epsilon = 0$ in Eq. (\ref{eq:DPmatrix}) as the resulting corruption scheme makes $p_{00} = p_{11} = 0.5$. On the other extreme, for $\epsilon= \infty$, feedbacks can be made equal to rewards.   

Using the corruption parameters from Eq. (\ref{eq:DPmatrix}) with  Corollary~\ref{UB_corollary}, we arrive  at the following upper bound.
\begin{myCorollary}
\label{DPUB_corollary}
The regret of \KLUCBCF \ or \TSCF \ at time $T$ with $\epsilon$-locally differentially private bandit feedback corruption scheme is 
$$ \cRegret_T \leq \sum_{a = 2}^{K}\frac{2\log(T)}{ \Delta_a \big( \frac{e^\epsilon - 1}{e^\epsilon + 1}\big)^2} + O(\sqrt{\log{(T)}} ).$$

\end{myCorollary}
\noindent

The term $\big( \frac{e^\epsilon - 1}{e^\epsilon + 1}\big)^2$ in the above expression conveys the relationship of the regret with the level of local differential privacy symbolized by $\epsilon$. For low values of $\epsilon$,  $ \big( \frac{e^\epsilon - 1}{e^\epsilon + 1}\big) \approx \epsilon / 2 $. 
This is in-line with the regret of the stochastic bandit algorithms providing global DP given by \citet[Theorem 4 and 8]{DBLP:conf/uai/MishraT15} which have a multiplicative factor of O($\epsilon^{-1})$ or O($\epsilon^{-2})$.
\citet[Corollary 3.2 and Theorem 3.5]{tossou:aaai2016} provided a regret bound for a stochastic bandit algorithm achieving global DP with an additive factor of  O($\epsilon^{-1}$). Our lower bound, given in Theorem~\ref{thm:LB}, shows that such an improvement is not expected for local differential privacy as parameters of the corruption mechanism are featured in the (asymptotic) multiplicative factor of $\log(T)$. It is also worthwhile to recall that local differential privacy comes at a higher price for the user : as local DP is a more stringent privacy notion than global DP, it is justifiable that the regret of the algorithms providing the latter is lower than that of the algorithms providing the former.

\section{Empirical Evaluation}
\label{sec:eval}
Before delving into the empirical evaluation, we first describe a naive algorithm called, \textsc{Wrapper} to be used as a baseline. This algorithm simply applies the appropriate inverse corruption function to the received feedback values and uses the result as a substitute for empirical reward. It then treats the corrupt bandit problem as a classical MAB problem and solves it using any classical MAB algorithm as a black-box. 
It is easy to see that this naive algorithm won't work for the corruptions functions in which $\bE(g_{.}^{-1}(y)) \neq g_{.}^{-1} (\bE (y))$. Even while using linear corruption functions
, this algorithm gives worse performance than the algorithms provided in this article, as can be verified below. The inferior performance is because this naive algorithm doesn't take into account the variance of the sequence generated by applying inverse corruption functions to the received feedback values.  

We provide here the evaluation of the algorithms on a 10-armed Bernoulli corrupt bandit problem. The reward means of the arms were set as follows:
$$ \mu_1 = 0.9 \qquad \mu_2 = \mu_3  = \dots = \mu_{10} = 0.8 $$ 
Further experiments can be found in Appendix~\ref{sec:eval2}.
\subsection{Regret over a period of time}
\label{sec:expt1}
In this experiment, we aim to see the effect of time on the regret of \KLUCBCF \ and \TSCF. Randomized response was employed to corrupt the feedback and according to Eq. (\ref{eq:CorruptionMatrix}), $p_{00} = p_{11} = 0.6$ for the optimal arm, while for all the other arms, both $p_{00}$ and $p_{11}$ were set to $0.9$. The time horizon was varied to $10^5$ and each experiment was repeated 1000 times. As a baseline, we plot the regret curves for two instances of the \textsc{Wrapper} algorithm (denoted as WR) with \KLUCB \ and \TS \ used as the black-box subroutine respectively. To demonstrate the inability of the traditional MAB algorithms to solve the corrupt bandit problem,
we also include the regret curves for \KLUCB\ ,\UCB\ and \TS\ (treating feedback as reward). The regret curves for all the considered algorithms are given in Figure~\ref{Fig1A}. LB denotes the lower bound given by Theorem~\ref{thm:LB}. The performance superiority of the proposed algorithms for corrupt bandits is more pronounced as the time increases. 
\subsection{Regret with varying level of local differential privacy}
\label{sec:expt2}
In this experiment, we vary the local differential privacy parameter and examine the effect on the regret of \KLUCBCF \ and \TSCF. We chose $\epsilon$ from the set $\{\frac{1}{8}, \frac{1}{4}, \frac{1}{2}, 1, 2, 4, 8\}$. The corruption parameters are set by substituting the values of $\epsilon$ in Eq. ~(\ref{eq:DPmatrix}). The time horizon was fixed to $10^5$ and the experiment was repeated $1000$ times. The corresponding curves for average regret can be seen in Figure~\ref{Fig1B}. UB indicates the upper bound given by Corollary~\ref{DPUB_corollary}. The regret for both the algorithms decreases with increasing $\epsilon$. This behavior is expected since, lower the value of $\epsilon$, more stringent is the level of differential privacy.  Towards both the end points of the range ( $\epsilon < 1/4$ and $\epsilon > 4$ ), the regret tends to plateau as a change in $\epsilon$ causes an infinitesimal change in the required level of differential privacy.

\begin{figure}
\centering
\begin{subfigure}{.3\textwidth}
  \centering
  \includegraphics[width=\linewidth]{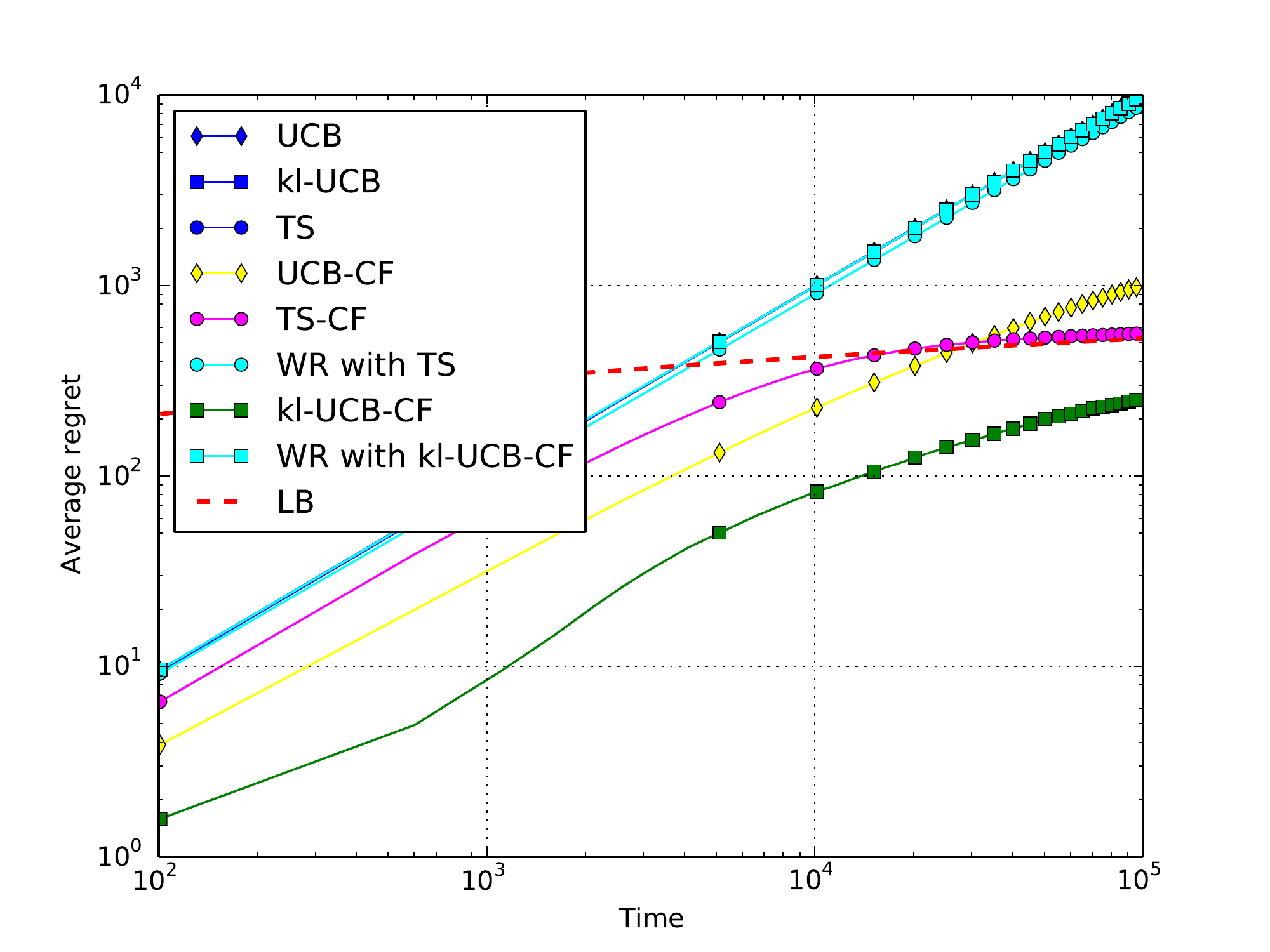}
  \caption{Regret plots for \KLUCBCF \ and \TSCF \ with others for varying  horizons up to $10^5$}
  \label{Fig1A}
\end{subfigure}%
\hspace{1em}
\begin{subfigure}{.3\textwidth}
  \centering
  \includegraphics[width=\linewidth]{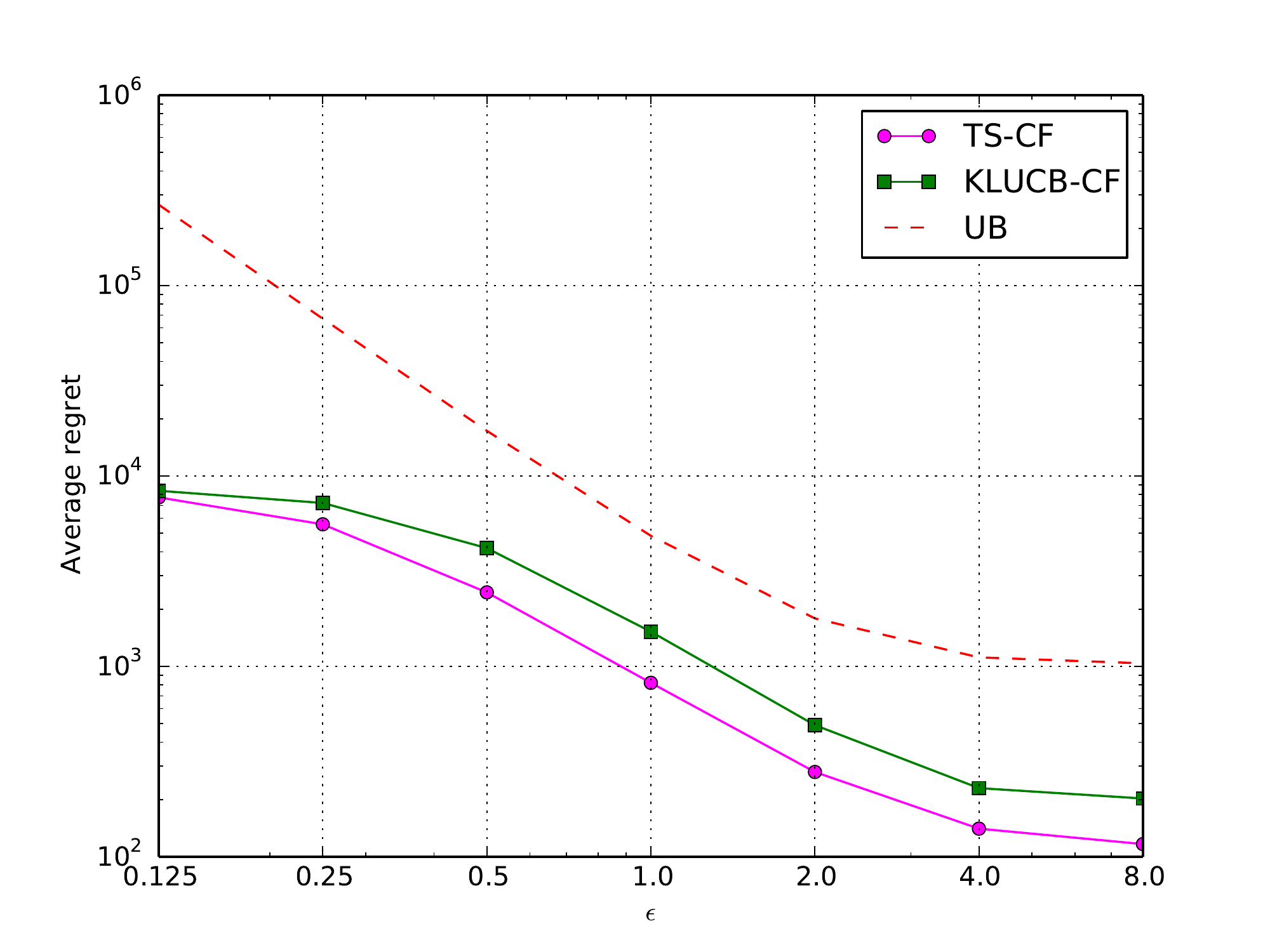}
  \caption{Regret plots for \KLUCBCF \ and \TSCF \ with $\epsilon$ = $\{\frac{1}{8}, \frac{1}{4}, \frac{1}{2}, 1, 2, 4, 8\}$, $T=10^5$}
  \label{Fig1B}
\end{subfigure}
\begin{subfigure}{.3\textwidth}
  \centering
  \includegraphics[width=\linewidth]{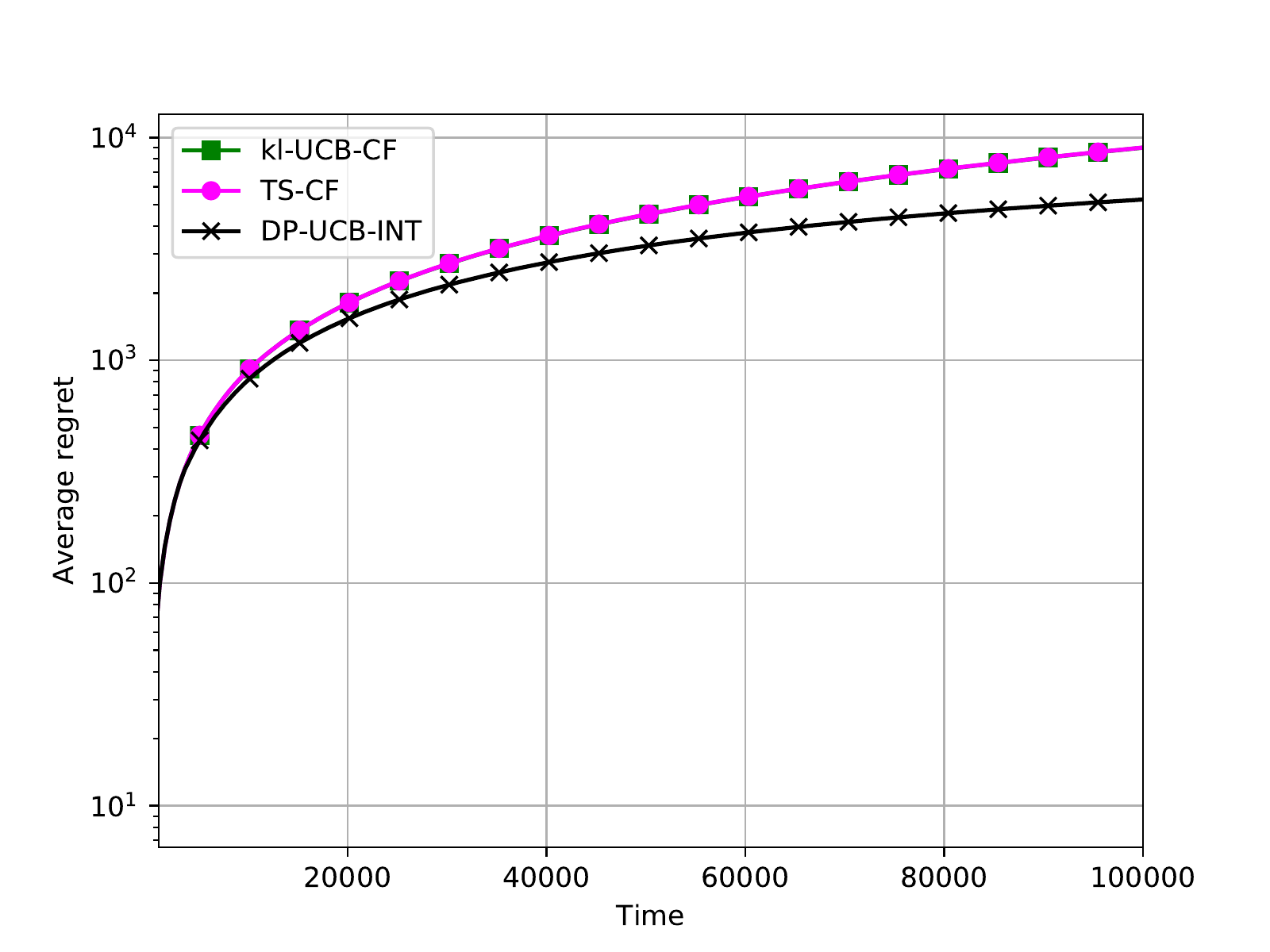}
  \caption{Regret plots for $\epsilon = 1$ for  \KLUCBCF \ and \TSCF \,  \textsc{DP-UCB-Int}}
  \label{Fig1C}
\end{subfigure}
\caption{Regret curves}
\label{Fig1}
\end{figure}

\subsection{Regret for local and global differential privacy}
\label{sec:expt3}
For comparison, we plot the regret of \KLUCBCF \ and \TSCF \ against the recent stochastic bandit algorithm for global DP, \textsc{DP-UCB-Int}, provided by \citet{tossou:aaai2016}. The comparison aims to convey how much utility, in terms of regret, is lost by opting for local DP instead of global DP. As already mentioned, lower regret for achieving global DP is to be expected as local DP is a much stronger notion of privacy than global DP. For \textsc{DP-UCB-Int}, we chose the same values of the algorithm parameters ($\delta = e^{-10}$ and $v=1.1$) as in the experiments given in \citet[Section 4]{tossou:aaai2016}. We provide the results for $\epsilon = 1$ in Figure \ref{Fig1C}.

\section{Elements of proofs} \label{sec:sketches}
We denote by $\hat{\lambda}_a(t)$ the empirical mean of the feedback obtained from arm $a$ until time $t$. Letting $F_{a,s}$ being the successive feedbacks of arm $a$ and $\hat{\lambda}_{a,s}:=\frac{1}{s}\sum_{\ell = 1}^s F_{a,\ell}$, one has $\hat{\lambda}_a(t) = \hat{\lambda}_{a,N_a(t)}$ when $N_a(t)>0$.

\subsection{Proof of Theorem~\ref{thm:LB}}
\label{sec:Proof_LB_Theorem}

To obtain a lower bound on the regret, we use a \textit{change-of-distribution} argument. Let $\boldsymbol{\nu} $ and $\boldsymbol{\nu'}$ be $K$-armed corrupted bandit models with different optimal arms i.e. $a_*(\boldsymbol{\nu}) \neq a_*(\boldsymbol{\nu'})$. For the ease of readability, let's assume without loss of generality that $a_*(\boldsymbol{\nu}) = 1$. 

The log-likelihood ratio of the observations up to time $T$ under $\boldsymbol{\nu}$ and $\boldsymbol{\nu}$, $L_T(\boldsymbol{\nu},\boldsymbol{\nu'})$, can be written 
\[L_T(\boldsymbol{\nu},\boldsymbol{\nu'}) =  \sum_{a = 1}^K \sum_{s \gets 1}^{N_a(T)} \log{  \frac{ f_{\lambda_a^{\boldsymbol{\nu}}} (F_{a,s})} { f_{\lambda_a^{\boldsymbol{\nu'}}}(F_{a,s}) } } \]
where $f_x(\cdot)$ denotes the Bernoulli density of mean $x$. Note that this likelihood ratio only features the \emph{feedback} distributions, from which we collect the observations. By Wald's lemma, $\bE_{\boldsymbol{\nu}}\left[L_T(\boldsymbol{\nu},\boldsymbol{\nu'})\right] = \sum_{a \gets1}^K\mathbb{E}_{\boldsymbol{\nu}} [N_a(T)] \cdot d(\lambda_a^{\boldsymbol{\nu}},\lambda_a^{\boldsymbol{\nu'}})$.

The following lemma can be extracted from \cite{garivier163}. 
\begin{myLemma}
\label{lem_change_dist} 
Let $\boldsymbol{\nu}$ and $\boldsymbol{\nu}'$ be two bandit models with $K$ arms and and $T \in \{0\} \cup \mathbb{N}$, then:
$$\sum_{a = 1}^{K} \mathbb{E}_{\boldsymbol{\nu}} [N_a(T)]\cdot{}KL(\lambda_a^{\boldsymbol{\nu}}, \lambda_a^{\boldsymbol{\nu}'}) \geq  d \left(\mathbb{E}_{\boldsymbol{\nu}}(Z), \mathbb{E}_{\boldsymbol{\nu}'}(Z)\right)$$
where $d(x,y) :=x\log(x/y)+(1-x)\log((1-x)/(1-y))$ is the binary relative entropy and $Z \in [0,1]$ is a random variable measurable from the past-observations filtration $(\mathcal{F}_T)$ 
\end{myLemma}
\noindent
Using Lemma \ref{lem_change_dist} with $Z \defined \frac{N_1(T)}{T}$, one obtains
\begin{equation}
\sum_{a = 1}^{K} \mathbb{E}_\nu(N_a(T)) \cdot d\left(\lambda_a^\nu, \lambda_a^{\nu'}\right)  \geq d \Big( \frac{\mathbb{E}_\nu (N_1(T))}{T} ,  \frac{\mathbb{E}_{\nu'} (N_1(T))}{T} \Big)
\label{eq:LBLem1}
\end{equation}

\noindent
Using the inequality $d(p,q) \geq p \log ({1}/{q}) - \log(2)$
(see \cite{garivier163}) yields 
\begin{equation*}
  d \Big( \frac{\mathbb{E}_{\boldsymbol{\nu}} (N_1(T))}{T} ,  \frac{\mathbb{E}_{\boldsymbol{\nu'}} (N_1(T))}{T} \Big)  \geq \frac{\mathbb{E}_{\boldsymbol{\nu}} (N_1(T))}{T} \log \Big( \frac{T}{\mathbb{E}_{\boldsymbol{\nu'}} (N_1(T))} \Big) - \log(2)
\end{equation*}

\noindent
Since $a_*(\boldsymbol{\nu}) = 1$, and $a_*(\boldsymbol{\nu'}) \neq 1$, $\mathbb{E}_{\boldsymbol{\nu}} (N_1(T)) \sim T$ and $\mathbb{E}_{\boldsymbol{\nu'}} (N_1(T)) = o(T^\alpha)$ for all $\alpha \in ]0,1]$. Hence one can show that 
\[ \frac{\mathbb{E}_{\boldsymbol{\nu}} (N_1(T))}{T} \sim 1 \ \ \text{and} \ \ \ \log\left( \frac{T}{\mathbb{E}_{\nu'} [N_1(T)]}\right) \sim \log(T).
\]
Equation~\eqref{eq:LBLem1} yields 
\begin{equation}
\liminf_{T\rightarrow \infty}\frac{\sum_{a \gets1}^{K} \mathbb{E}_{\boldsymbol{\nu}} (N_a(T)) \cdot d\left(\lambda_a^{\boldsymbol{\nu}}, \lambda_a^{\boldsymbol{\nu'}}\right)}{\log{T}} \geq 1. 
\label{eq:LBME}
\end{equation}

\noindent
To obtain a lower bound on $\bE_{\boldsymbol{\nu}}[N_a(T)]$ for each $a \in \{2,\dots,K\}$, one can choose $\boldsymbol{\nu'}$ such that, for some $\epsilon > 0$,
\begin{equation*}
\mu^{\boldsymbol{\nu'}}_{b} = 
\begin{cases}
\mu^{\boldsymbol{\nu}}_{1} + \epsilon, & \text{if}\ b \ \IsEqual a \\
\mu^{\boldsymbol{\nu}}_{b} & \text{otherwise}
\end{cases}
\end{equation*}
This translates to the following change in feedback,
 \begin{equation*}
    \lambda^{\boldsymbol{\nu'}}_b = 
    \begin{cases}
       g_b(\mu^{\boldsymbol{\nu}}_{1} + \epsilon)   & \text{if}\ b \IsEqual a, \\
      g_b(\mu^{\boldsymbol{\nu}}_{b}) = \lambda_b^{\boldsymbol{\nu}} & \text{otherwise}.
    \end{cases}
  \end{equation*}
As $d\left(\lambda_b^{\boldsymbol{\nu}}, \lambda_b^{\boldsymbol{\nu'}}\right) = 0$ for $b\neq a$, using equation (\ref{eq:LBME}) we get
\[
\liminf_{T \to \infty} \frac{\mathbb{E}_{\boldsymbol{\nu}}(N_a(T))}{\log{T}} \geq \frac{1}{d\left(\lambda_a^{\boldsymbol{\nu}},g_a(\mu_1 + \epsilon)\right)}
\]
Letting $\epsilon$ go to zero for each $a \in \{2,\dots,K\}$ (and assuming $\{g_a\}_{a \in A }$ are continuous), one obtains,
\[
\liminf_{T\rightarrow \infty}\frac{\cRegret_T(\boldsymbol{\nu})}{\log(T)} \geq \sum_{a \gets2}^{K} \frac{\Delta_a^{\boldsymbol{\nu}}}{d\left( \lambda_a^{\boldsymbol{\nu}}, g_a(\mu_1^{\boldsymbol{\nu}})\right)}.
\]

\subsection{Proof outline for Theorem~\ref{UB_KLUCB_Theorem}}
\label{Sksec:Proof_UB_KLUCBCF_Theorem}

We defer the complete proof of Theorem~\ref{UB_KLUCB_Theorem} to Appendix~\ref{sec:Proof_UB_KLUCBCF_Theorem}. In this subsection, we describe the road-map for the proof. We arrive at a upper bound on the regret of \KLUCBCF \ by first bounding the number of times any suboptimal arm $a$ is pulled by the algorithm till horizon $T$, $\bE[N_a(T)]$. Recall that, at any time \KLUCBCF \ pulls an arm maximizing a index defined as
\begin{align*}
\mathrm{Index}_a(t) &\defined \max 
\left\{ q:\  N_a(t)\cdot{}d(\hat{\lambda}_a(t), g_a(q)) \leq f(t)\right\} \\ 
&= 
 \max g_a^{-1}\left( \{q: N_a(t)\cdot{}d(\hat{\lambda}_a(t), q) \leq f(t)\} \right)
\end{align*}
Figure \ref{fig:KLindices} depicts the computation process for the index in \KLUCBCF \ and how it differs from the index computed by \KLUCB \ to account for the presence of corruption.
\begin{figure}
\centering
\includegraphics[width=0.6\linewidth]{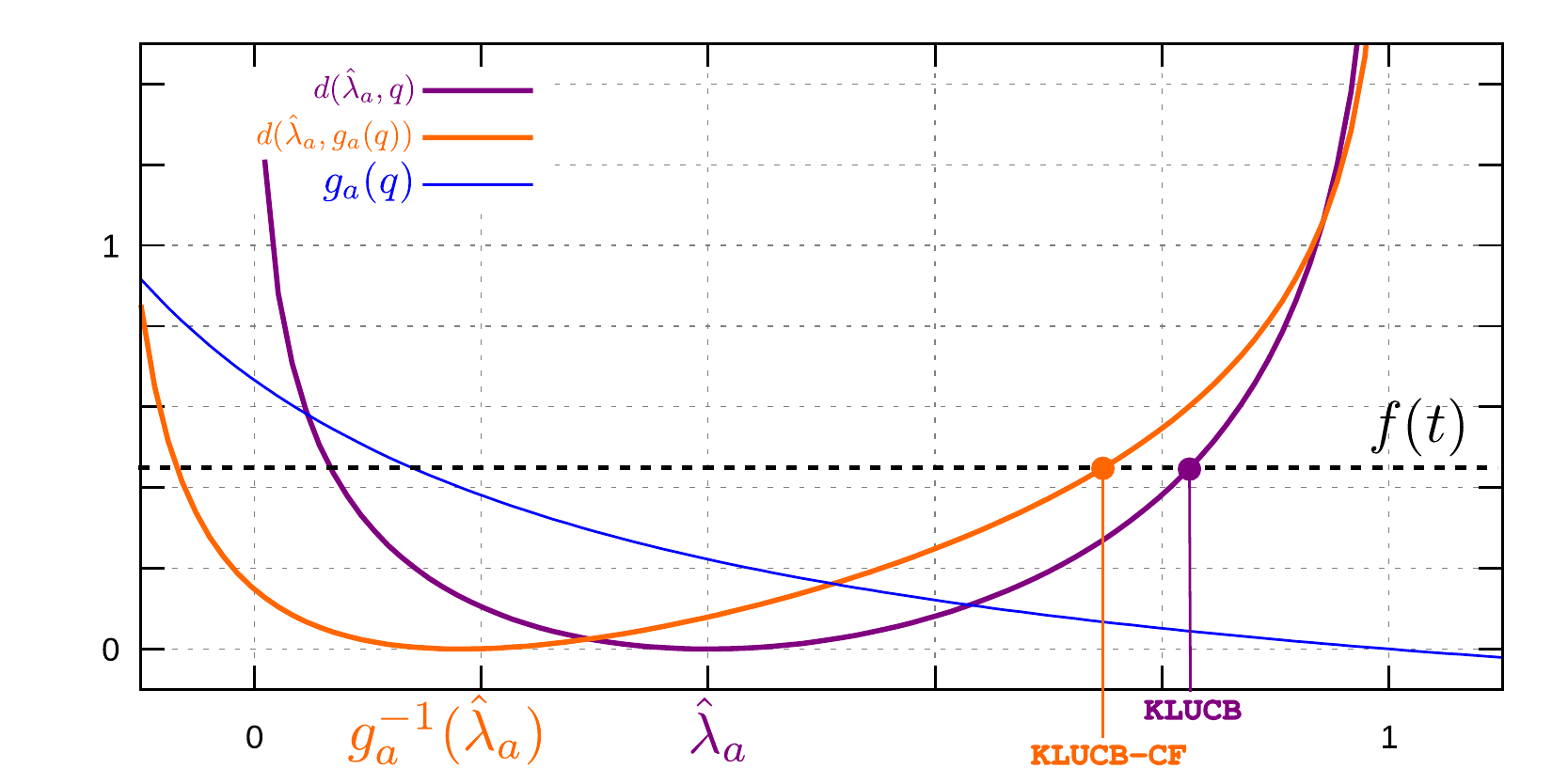}
\caption[width=.7\linewidth]{KL indices calculation.}
\label{fig:KLindices}
\end{figure}
\paragraph{}
For the purpose of this proof, we further decompose the index computation as follows:
$$\mathrm{Index}_a(t) \defined
    \begin{cases}
      g_a^{-1}({\ell_a(t)}) & \text{if } g_a \text{ is decreasing}, \\
      g_a^{-1}({u_a(t)}) & \text{if } g_a \text{ is increasing,}
    \end{cases}
$$
with
$$
\ell_a(t)   \defined   \min \{q: N_a(t)\cdot{}d(\hat{\lambda}_a(t), q) \leq f(t)\} \quad \text{and} \quad
u_a(t)  \defined   \max \{q: N_a(t)\cdot{}d(\hat{\lambda}_a(t), q) \leq f(t)\}.
$$
The interval $[\ell_a(t), u_a(t)]$ is a KL-based confidence interval on the mean feedback $\lambda_a$ of arm $a$. This is in contrast to the analysis of the \KLUCB\ algorithm given by \citet{KLUCBJournal} where a confidence interval is placed on the mean reward of arm $a$.

In our analysis, we use the fact that when arm $a$ is picked at time $t+1$ by \KLUCBCF , one of the following is true. Either the mean feedback of the optimal arm $1$ is outside its confidence interval (i.e. $g_1(\mu_1) < \ell_1(t)$ or $g_1(\mu_1) > u_1(t)$), which is unlikely, or the mean feedback of the optimal arm is where it should be, and then the fact that arm $a$ is selected indicates that the confidence interval on $\lambda_a$ cannot be too small as either $(u_a(t) \geq g_a(\mu_1))$ or $(\ell_a(t) \leq g_a(\mu_1))$. The previous statement follows from considering various cases depending on whether the corruption functions $g_a$ and $g_1$ are increasing or decreasing. 

We then need to control the two terms in the decomposition of the expected number of draws of arm $a$. The term regarding the ``unlikely" event, is easily bounded using the same technique as in the \KLUCB \ analysis, and is of order $o(\log(T))$. To control the second term, depending on the monotonicity of the corruption functions $g_a$ and $g_1$, we need to adapt the arguments in \cite{KLUCBJournal} to control the number of draws of arm $a$, as can be seen in Appendix~\ref{sec:Proof_UB_KLUCBCF_Theorem}.

\subsection{Proof outline for Theorem~\ref{UB_TSCF_Theorem}}
\label{Sksec:Proof_UB_TSCF_Theorem}

Our proof follows the analysis of \citet{DBLP:conf/aistats/AgrawalG13} for classical Thompson Sampling. We proceed by controlling the number of draws of each suboptimal arm $a$. For this purpose, we introduce two thresholds $\LThresholdTS_a$ and $\HThresholdTS_a$ that satisfy $ \lambda_a < \LThresholdTS_a < \HThresholdTS_a < g_a(\mu_1)$ if $g_a$ is increasing and $ \lambda_a > \LThresholdTS_a > \HThresholdTS_a > g_a(\mu_1)$ if $g_a$ is decreasing. We introduce $E_a^\lambda(t)$ as the event $\{g_a^{-1} (\hat{\lambda}_a(t)) \leq g_a^{-1} (\LThresholdTS_a)\}$ and $E_a^\theta(t)$ as the event $\{g_a^{-1} (\theta_a(t)) \leq g_a^{-1} (\HThresholdTS_a)\}$. We then upper bound $\bE[N_a(T)]$ by the sum of the three terms as,

\vspace{-0.6cm}

\[
\sum_{t=0}^{T-1} \bP(\ArmAt_{t+1}=a, E_a^\lambda(t), \overline{E_a^\theta(t)}) + \sum_{t=0}^{T-1} \bP(\ArmAt_{t+1}=a, E_a^\lambda(t), E_a^\theta(t)) + \sum_{t=0}^{T-1} \bP(\ArmAt_{t+1}=a, \overline{E_a^\lambda(t)}).
\]

\vspace{-0.3cm}

Using arguments similar to \cite{DBLP:conf/aistats/AgrawalG13}, with some adaptations, we then show that the last two terms are of order $o(\log(T))$. To control the first term, we prove the following, which requires some extra technicalities compared to the original proof, as shall be seen in  Appendix~\ref{sec:Proof_UB_TSCF_Theorem}, where the full proof of Theorem~\ref{UB_TSCF_Theorem} is given. 
\begin{restatable}{myLemma}{TSLMC} When $g_a$ is increasing (resp. decreasing), for any $\LThresholdTS_a' \in (\LThresholdTS_a,\HThresholdTS_a)$ (resp. $(\HThresholdTS_a,\LThresholdTS_a)$),  

\vspace{-0.4cm}

\[\sum_{t \gets 0 }^{T-1}\bP\Big(\ArmAt_{t+1} \IsEqual a, \overline{E_a^\theta(t)}, E_a^\lambda(t)\Big) \leq \frac{\log(T)}{d(\LThresholdTS_a',\HThresholdTS_a)} + 1 \ \ \ \text{when } T \text{ is large enough}.\] 
\end{restatable}

\section{Conclusion}
\label{sec:conclusion}
Both the algorithms introduced in this article, \KLUCBCF \ and \TSCF \ provide suitable solutions to the MAB problem with corrupted feedback, as they are proved to asymptotically attain the best possible (problem-dependent) regret. Our experiments confirm the theoretical analysis by demonstrating the superior performance of \KLUCBCF \ and \TSCF{}. Furthermore, we exhibit appropriate corruption matrices that achieve a desired level of local differential privacy, and quantify their impact on the regret. These algorithms are thus good candidates to be used in recommender systems which apply a randomized response mechanism to protect the user privacy.

This work can be extended in many ways. In our setting, although the feedback is corrupted, it is available at all times. In some situations however, the feedback is simply lost. As future work, we plan to extend our problem setting to incorporate such scenarios by making appropriate changes to the corruption process. 
An adversarial corruption of the feedback can be considered too. Another possible extension is to incorporate contextual information in the learning process. 
We conjecture that the invertibility condition on the corruption functions can be relaxed for \KLUCBCF{} as long as $\lambda_a \neq g_a(\mu_1)$ for all suboptimal arms but it remains to be proven.  



\newpage

\bibliographystyle{plainnat}
\small
\bibliography{references}
\clearpage
\appendix
\renewcommand{\thesection}{\Roman{section}}
\renewcommand\thesubsection{\Alph{subsection}}

\normalsize

\paragraph{Notations.} For the proofs, we recall that $\hat{\lambda}_a(t)$ is the empirical mean of the feedback obtained from arm $a$ until time $t$. Letting $F_{a,s}$ being the successive feedbacks of arm $a$ and $\hat{\lambda}_{a,s}:=\frac{1}{s}\sum_{\ell = 1}^s F_{a,\ell}$, one has $\hat{\lambda}_a(t) = \hat{\lambda}_{a,N_a(t)}$ when $N_a(t)>0$.

\section{Proof for Theorem~\ref{UB_KLUCB_Theorem}}
\label{sec:Proof_UB_KLUCBCF_Theorem}
\begin{proof}
The index is defined by
$$\mathrm{Index}_a(t) \defined \max 
\left\{ q:\  N_a(t)\cdot{}d(\hat{\lambda}_a(t), g_a(q)) \leq f(t)\right\} = 
 \max g_a^{-1}\left( \{q: N_a(t)\cdot{}d(\hat{\lambda}_a(t), q) \leq f(t)\} \right)
$$
For the purpose of this proof, we further decompose the computation of index as follows,
$$\mathrm{Index}_a(t) \defined
    \begin{cases}
      g_a^{-1}({\ell_a(t)}) & \text{if } g_a \text{ is decreasing}, \\
      g_a^{-1}({u_a(t)}) & \text{if } g_a \text{ is increasing}
    \end{cases}
$$
where,
$$
\ell_a(t)  \defined  \min \{q: N_a(t)\cdot{}d(\hat{\lambda}_a(t), q) \leq f(t)\} \text{ and }
u_a(t)  \defined  \max \{q: N_a(t)\cdot{}d(\hat{\lambda}_a(t), q) \leq f(t)\}
$$

To get an upper bound on the regret of this algorithm, we first bound $\mathbb{E}[N_a(t)]$ for all the non-optimal arms $a$. Note that, we assume $1$ to be the optimal arm. 
\begin{align*}
\mathbb{E}(N_a(T)) &= 1 + \sum_{t \gets K}^{T-1}\mathbb{P}(\ArmAt_{t+1} \IsEqual a)\end{align*}
Depending upon if $g_a$ and $g_1$ are increasing or decreasing there are four possible sub-cases:
\begin{itemize}
\item Both $g_1$ and $g_a$ are increasing.
\begin{align*}
&(\ArmAt_{t+1} \IsEqual a) \\
&\subseteq (u_1(t) < g_1(\mu_1)) \cup (\ArmAt_{t+1} \IsEqual a, u_1(t) \geq g_1(\mu_1)) \\
&= (u_1(t) < g_1(\mu_1)) \cup (\ArmAt_{t+1} \IsEqual a, g^{-1}_1(u_1(t)) \geq \mu_1) \qquad \text{since $g_1$ is increasing}\\
&= (u_1(t) < g_1(\mu_1)) \cup (\ArmAt_{t+1} \IsEqual a, g^{-1}_a(u_a(t)) 
\geq \mu_1) \qquad \text{since $\operatorname{Index}_a>\operatorname{Index}_1$}\\
&= (u_1(t) < g_1(\mu_1)) \cup (\ArmAt_{t+1} \IsEqual a, u_a(t) \geq g_a(\mu_1)) \qquad \text{since $g_a$ is increasing}
\end{align*}
\begin{equation}
\label{exp_of_N_eq1}
\therefore \mathbb{E}(N_a(T)) \leq 1 + \sum_{t \gets K}^{T-1} \mathbb{P}(u_1(t) < g_1(\mu_1)) + \sum_{t \gets K}^{T-1} \mathbb{P}(\ArmAt_{t+1} \IsEqual a, u_a(t) \geq g_a(\mu_1))
\end{equation}
\item $g_1$ is decreasing and $g_a$ is increasing.
\begin{align*}
&(\ArmAt_{t+1} \IsEqual a) \\
&\subseteq (\ell_1(t) > g_1(\mu_1)) \cup (\ArmAt_{t+1} \IsEqual a, \ell_1(t) \leq g_1(\mu_1)) \\
&= (\ell_1(t) > g_1(\mu_1)) \cup (\ArmAt_{t+1} \IsEqual a, g_1^{-1}(\ell_1(t)) \geq \mu_1) \qquad \text{since $g_1$ is decreasing} \\
&= (\ell_1(t) > g_1(\mu_1)) \cup (\ArmAt_{t+1} \IsEqual a, g_a^{-1}(u_a(t)) \geq \mu_1) \qquad \text{since $\operatorname{Index}_a>\operatorname{Index}_1$}\\
&= (\ell_1(t) > g_1(\mu_1)) \cup (\ArmAt_{t+1} \IsEqual a, u_a(t) \geq g_a(\mu_1)) \qquad \text{since $g_a$ is increasing}
\end{align*}
\begin{equation}
\label{exp_of_N_eq2}
\therefore \mathbb{E}(N_a(T)) \leq 1 + \sum_{t \gets K}^{T-1} \mathbb{P}(\ell_1(t) > g_1(\mu_1))  +  \sum_{t \gets K}^{T-1} \mathbb{P}(\ArmAt_{t+1} \IsEqual a, u_a(t) \geq g_a(\mu_1)) 
\end{equation}
\item $g_1$ is increasing and $g_a$ is decreasing.
\begin{align*}
&(\ArmAt_{t+1} \IsEqual a) \\
&\subseteq (u_1(t) < g_1(\mu_1)) \cup (\ArmAt_{t+1} \IsEqual a, u_1(t) \geq g_1(\mu_1)) \\
&= (u_1(t) < g_1(\mu_1)) \cup (\ArmAt_{t+1} \IsEqual a, g_1^{-1}(u_1(t)) \geq \mu_1) \qquad \text{since $g_1$ is increasing}\\ 
&= (u_1(t) < g_1(\mu_1)) \cup (\ArmAt_{t+1} \IsEqual a, g_a^{-1}(\ell_a(t)) \geq \mu_1) \qquad \text{since $\operatorname{Index}_a>\operatorname{Index}_1$}\\ 
&= (u_1(t) < g_1(\mu_1)) \cup (\ArmAt_{t+1} \IsEqual a, \ell_a(t) \leq g_a(\mu_1)) \qquad \text{since $g_a$ is decreasing}
\end{align*}
\begin{equation}
\label{exp_of_N_eq3}
\therefore \mathbb{E}(N_a(T)) \leq 1 +  \sum_{t \gets K}^{T-1} \mathbb{P}(u_1(t) < g_1(\mu_1)) +  \sum_{t \gets K}^{T-1} \mathbb{P}(\ArmAt_{t+1} \IsEqual a, \ell_a(t) \leq g_a(\mu_1))
\end{equation}
\item $g_1$ is decreasing and $g_a$ is decreasing.
\begin{align*}
&(\ArmAt_{t+1} \IsEqual a) \\
&\subseteq (\ell_1(t) > g_1(\mu_1)) \cup (\ArmAt_{t+1} \IsEqual a, \ell_1(t) \leq g_1(\mu_1)) \\
&= (\ell_1(t) > g_1(\mu_1)) \cup (\ArmAt_{t+1} \IsEqual a, g_1^{-1}(\ell_1(t)) \geq \mu_1) \qquad \text{since $g_1$ is decreasing} \\
&= (\ell_1(t) > g_1(\mu_1)) \cup (\ArmAt_{t+1} \IsEqual a, g_a^{-1}(\ell_a(t)) \geq \mu_1) \qquad \text{since $\operatorname{Index}_a>\operatorname{Index}_1$} \\
&= (\ell_1(t) > g_1(\mu_1)) \cup (\ArmAt_{t+1} \IsEqual a, \ell_a(t) \leq g_a(\mu_1)) \qquad \text{since $g_a$ is decreasing}
\end{align*}
\begin{equation}
\label{exp_of_N_eq4}
\therefore \mathbb{E}(N_a(T)) \leq 1 + \sum_{t \gets K}^{T-1} \mathbb{P}(\ell_1(t) > g_1(\mu_1)) +  \sum_{t \gets K}^{T-1} \mathbb{P}(\ArmAt_{t+1} \IsEqual a, \ell_a(t) \leq g_a(\mu_1)) 
\end{equation}
\end{itemize}

We first upper bound the two sums 
\begin{equation}\sum_{t \gets K}^{T-1}\mathbb{P}(u_1(t) < g_1(\mu_1)) \ \ \text{and} \ \ \sum_{t \gets K}^{T-1}\mathbb{P}(\ell_1(t) > g_1(\mu_1))\label{eq:FirstTerms}\end{equation}
using that  $\ell_1(t)$ and $u_1(t)$ are respectively lower and upper confidence bound on $g_1(\mu_1)$. Indeed,
\begin{eqnarray*}
\bP(u_1(t) < g_1(\mu_1)) & \leq & \bP\left(g_1(\mu_1) > \hat{\lambda}_1(t) \text{ and } N_1(t)d(\hat{\lambda}_1(t), g_1(\mu_1)) \geq {f(t)}\right) \\
& \leq &\bP\left(\exists s \in \{1,\dots,t\} :  g_1(\mu_1) > \hat{\lambda}_{1,s} \text{ and } sd(\hat{\lambda}_{1,s}, g_1(\mu_1)) \geq {f(t)}\right) \\
& \leq & min\{ 1, e \lceil f(t) \log{t} \rceil e^{-f(t)} \},
\end{eqnarray*}
where the upper bound follows from Lemma 2 in \cite{KLUCBJournal}, and the fact that $\hat{\lambda}_{1,s}$ is the empirical mean of $s$ Bernoulli samples with mean $g_1(\mu_1)$.
Similarly, one has 
\begin{eqnarray*}
\bP(\ell_1(t) > g_1(\mu_1))  \leq  \min\{ 1, e \lceil f(t) \log{t} \rceil e^{-f(t)} \}.
\end{eqnarray*}
As $f(t) \defined \log{t} + 3(\log{\log{t}})$ for $t\geq3$,
\[e \lceil f(t) \log{t} \rceil \leq 4e \log^2{t},\] the two quantities in \eqref{eq:FirstTerms} can be upper bounded by 
\begin{align*}
1 + \sum_{t \gets3}^{T-1} e \lceil f(t) \log{t} \rceil e^{-f(t)} & \leq 
 1 + \sum_{t \gets3}^{T-1} 4e \cdot \log^2{t} \cdot e^{-f(t)} \nonumber \\
&= 1 + 4e\sum_{t \gets3}^{T-1} \frac{1}{t \log{t}} \nonumber \\
&\leq 4e\Big ( \frac{1}{3 \log{3}} + \int_{3}^{T-1} \frac{1}{t\log{t}} dt \Big ) \nonumber \\
&\leq 4e \Big(\frac{1}{3 \log{3}} + \log{(\log{(T-1)})} - \log{(\log{3})} \Big) \nonumber \\
&\leq 3 + 4e \log{(\log{T})}. 
\end{align*}
This proves that 
\begin{eqnarray}
\sum_{t \gets K}^{T-1}\mathbb{P}(u_1(t) < g_1(\mu_1)) &\leq & 3 + 4e \log{(\log{T})} 
\in o(\log{T}) \label{term1a} \\
\sum_{t \gets K}^{T-1}\mathbb{P}(\ell_1(t) > g_1(\mu_1)) &\leq & 3 + 4e \log{(\log{T})} \in o(\log{T})\label{term1b}
\end{eqnarray}

We now turn our attention to the other two sums involved in the upper bound we gave for $\mathbb{E}(N_a(t))$. We introduce the notation $d^+(x,y) = d(x,y)\ind_{(x<y)}$ and $d^-(x,y) = d(x,y)\ind_{(x>y)}$. So we can write, when $g_a$ is increasing,
\begin{align*}
 & \sum_{t \gets K}^{T-1} \mathbb{P}(\ArmAt_{t+1} \IsEqual a, u_a(t) \geq g_a(\mu_1)) \\
 &= \bE\left[\sum_{t \gets K}^{T-1} \ind_{\ArmAt_{t+1} \IsEqual a} \ind_{N_a(t) \cdot d^+(\hat{\lambda}_{i,N_a(t)},g_a(\mu_1)) \leq f(t)}\right]  \\
 &\leq \bE\left[ \sum_{t \gets K}^{T-1} \sum_{s \gets 1}^{t} \ind_{\ArmAt_{t+1} \IsEqual a} \ind_{N_a(t) \IsEqual s} \ind_{s \cdot  d^+(\hat{\lambda}_{a,s},g_a(\mu_1)) \leq f(T)} \right] \\
 &= \bE\Big[\sum_{s \gets 1}^{T-1} \ind_{s \cdot  d^+(\hat{\lambda}_{a,s},g_a(\mu_1)) \leq f(T)} \underbrace{\sum_{s \gets 1}^{T-1} \ind_{\ArmAt_{t+1} \IsEqual a} \ind_{N_a(t) \IsEqual s }}_{\leq 1} \Big]. \end{align*}
One obtains, when $g_a$ is increasing, 
\begin{equation} \sum_{t \gets K}^{T-1} \mathbb{P}(\ArmAt_{t+1} \IsEqual a, u_a(t) \geq g_a(\mu_1))
\leq \sum_{s \gets 1}^{T-1} \bP\left( s \cdot d^+(\hat{\lambda}_{a,s},g_a(\mu_1)) \leq f(T)  \right).\label{eq:SecondTermIncreasing}\end{equation}
Using similar arguments, one can show that when $g_a$ is decreasing,
\begin{equation} \sum_{t \gets K}^{T-1} \mathbb{P}(\ArmAt_{t+1} \IsEqual a, \ell_a(t) \leq g_a(\mu_1))
\leq \sum_{s \gets 1}^{T-1} \bP\left( s \cdot d^-(\hat{\lambda}_{a,s},g_a(\mu_1)) \leq f(T)  \right).\label{eq:SecondTermDecreasing}\end{equation}
The quantity in the right-hand side of \eqref{eq:SecondTermIncreasing} is upper bounded in Appendix A.2. of \cite{KLUCBJournal} by 
\begin{equation} \frac{f(T)}{d(\lambda_a,g_a(\mu_1)}
+ \sqrt{2\pi}\sqrt{\frac{(d'(\lambda_a,g_a(\mu_1)))^2}{(d(\lambda_a,g_a(\mu_1)))^3}}\sqrt{f(T)} + 2 \left(\frac{d'(\lambda_a,g_a(\mu_1))}{d(\lambda_a,g_a(\mu_1))}\right)^{2} +1. \label{eq:SecondTermIncreasingFinal}\end{equation}
For the second term, noting that $d^-(x,y) = d^+(1-x,1-y)$, one has 
\begin{eqnarray*}\bP\left( s \cdot d^-(\hat{\lambda}_{a,s},g_a(\mu_1)) \leq f(T)  \right) &=& \bP\left( s \cdot d^+(1-\hat{\lambda}_{a,s},1-g_a(\mu_1)) \leq f(T)  \right)\\
&=& \bP\left( s \cdot d^+(\hat{\mu}_{a,s},1-g_a(\mu_1)) \leq f(T)  \right),\end{eqnarray*}
where $\hat{\mu}_{a,s} \defined 1-\hat{\lambda}_{a,s}$, is the empirical mean of $s$ observations of a Bernoulli random variable with mean $1-\lambda_a < 1 - g_a(\mu_1)$. Hence, the analysis of \cite{KLUCBJournal} can be applied, and using that $d(1-\lambda_a,1-g_a(\mu_1))=d(\lambda_a,g_a(\mu_1))$ and $d'(1-\lambda_a,1-g_a(\mu_1))=-d'(\lambda_a,g_a(\mu_1))$, the left hand side of \eqref{eq:SecondTermDecreasing} can also be upper bound by \eqref{eq:SecondTermIncreasingFinal}.

Combining inequalities \eqref{term1a}, \eqref{term1b} and \eqref{eq:SecondTermIncreasing},\eqref{eq:SecondTermDecreasing}, \eqref{eq:SecondTermIncreasingFinal} with the initial decomposition of $\bE[N_a(T)]$ yield in all cases 
\begin{eqnarray*}
\bE[N_a(T)] & \leq & \frac{\log(T)}{d(\lambda_a,g_a(\mu_1))} + \sqrt{2\pi}\sqrt{\frac{d'(\lambda_a,g_a(\mu_1))^2}{d(\lambda_a,g_a(\mu_1))^3}}\sqrt{\log(T) + 3\log\log(T)} \\
& & + \left(4e  + \frac{3}{d(\lambda_a,g_a(\mu_1)}\right)\log\log(T) + 2\left(\frac{d'(\lambda_a,g_a(\mu_1)}{d(\lambda_a,g_a(\mu_1)}\right)^2 + 4.
\end{eqnarray*}
Hence the regret of \KLUCBCF \ is upper bounded by 
$$ \sum_{a \gets2}^{K} \Delta_a \Big[ \frac{\log(T)}{D_a} + \sqrt{2\pi}\sqrt{\frac{(D'_a)^2}{D_a^3}}\sqrt{\log(T) + 3\log\log(T)}
 + \left(4e  + \frac{3}{D_a}\right)\log\log(T) + 2\left(\frac{D'_a}{D_a}\right)^2 + 4 \Big] $$ 
where $D_a \defined d(\lambda_a,g_a(\mu_1))$ and $D'_a \defined d'(\lambda_a,g_a(\mu_1))$, which concludes the proof.
\end{proof}

\section{Proof of Theorem~\ref{UB_TSCF_Theorem}}
\label{sec:Proof_UB_TSCF_Theorem}
\begin{proof}
Assume $1$ to be the optimal arm. For each arm non-optimal arm $a$, choose two thresholds $\LThresholdTS_a$ and $\HThresholdTS_a$ such that $ \lambda_a < \LThresholdTS_a < \HThresholdTS_a < g_a(\mu_1)$ if $g_a$ is increasing and $ \lambda_a > \LThresholdTS_a > \HThresholdTS_a > g_a(\mu_1)$ if $g_a$ is decreasing. Define $E_a^\lambda(t)$ as the event $\{g_a^{-1} (\hat{\lambda}_a(t)) \leq g_a^{-1} (\LThresholdTS_a)\}$ and $E_a^\theta(t)$ as the event $\{g_a^{-1} (\theta_a(t)) \leq g_a^{-1} (\HThresholdTS_a)\}$. Define $\mathcal{F}_{t}$ as the history of arm selections and received feedbacks including time $t$ and recall that \TSCF \ selects the arm as follows, $$\ArmAt_{t+1} \IsEqual \argmax_{a} \ \theta_a(t)$$,
where $\theta_a(t)$ is a sample from the posterior distribution on arm $a$ after $t$ observations. Define $p_{a,t} := \bP(g_1^{-1} (\theta_1(t)) > g_a^{-1}(\HThresholdTS_a) \given \mathcal{F}_{t})$. 

We start from the following decomposition.
\begin{align*}
\bE[N_a(T)] &= \sum_{t \gets 0 }^{T-1} \bP(\ArmAt_{t+1} \IsEqual a, E_a^\lambda(t), E_a^\theta(t)) + \sum_{t \gets 0 }^{T-1} \bP(\ArmAt_{t+1} \IsEqual a, E_a^\lambda(t), \overline{E_a^\theta(t)}) \\
&+ \sum_{t \gets 0 }^{T-1} \bP(\ArmAt_{t+1} \IsEqual a, \overline{E_a^\lambda(t)}) 
\end{align*}
Below are the lemmas that permit us to bound these three terms. These results generalize to the corrupted setting the main steps of the analysis of Thompson Sampling by \citet{DBLP:conf/aistats/AgrawalG13}. The proofs for these lemmas follow that of the corresponding lemmas in the aforementioned article, with some technicalities that arise from the fact that $g_1$ and $g_a$ may be either increasing or decreasing.

\begin{restatable}{myLemma}{TSLMA}
\label{TSCF-lemma1}
$
\bP(\ArmAt_{t+1} \IsEqual a, E_a^\theta(t), E_a^\lambda(t) \given \mathcal{F}_{t}) \leq \frac{(1 - p_{a,t})}{p_{a,t}} \bP(\ArmAt_{t+1} \IsEqual 1, E_a^\theta(t), E_a^\lambda(t) \given \mathcal{F}_{t})
$
\end{restatable}
\begin{proof}
Assume that $E_a^\lambda(t)$ is true (otherwise the lemma holds trivially because the left hand size is $0$). Hence, it is sufficient to prove that,
\begin{equation}
\label{eq:SufStmt}
\mathbb{P}(\ArmAt_{t+1} \IsEqual a \given E_a^\theta(t), \mathcal{F}_{t}) \leq \frac{(1 - p_{a,t})}{p_{a,t}} \mathbb{P}(\ArmAt_{t+1} \IsEqual 1\given E_a^\theta(t), \mathcal{F}_{t}))
\end{equation}

Define $M_a(t)$  the event in which the index of arm $a$ at time $t$ is the largest among those of all suboptimal arms: $M_a(t) := \left\{g_a^{-1}(\theta_a(t)) \geq g_j^{-1}(\theta_j(t)), \forall j \neq 1\right\}$.

\begin{align}
&\mathbb{P}(\ArmAt_{t+1} \IsEqual 1 \given E_a^\theta(t), \mathcal{F}_{t})) \nonumber \\
&\geq \mathbb{P}(\ArmAt_{t+1} \IsEqual 1, M_a(t) \given E_a^\theta(t), \mathcal{F}_{t})) \nonumber \\
&= \mathbb{P}(M_a(t) \given E_a^\theta(t), \mathcal{F}_{t}))  \cdot \mathbb{P}(\ArmAt_{t+1} \IsEqual 1 \given M_a(t), E_a^\theta(t), \mathcal{F}_{t})) \label{prob_At1_a}
\end{align}
Now, given $M_a(t)$ and $E_a^\theta(t)$ hold,
$$ g_j^{-1}(\theta_j(t)) \leq g_a^{-1}(\theta_a(t)) \leq g_a^{-1} (\HThresholdTS_a) \quad \forall j \neq a, j \neq 1 $$ 
So,
\begin{align}
 \mathbb{P}(\ArmAt_{t+1} \IsEqual 1 \given M_a(t), E_a^\theta(t), \mathcal{F}_{t}) &\geq \mathbb{P}( g_1^{-1} (\theta_1(t)) > g_a^{-1}(\HThresholdTS_a) \given M_a(t), E_a^\theta(t), \mathcal{F}_{t}) \nonumber \\
&= \mathbb{P}( g_1^{-1} (\theta_1(t)) > g_a^{-1}(\HThresholdTS_a) \given \mathcal{F}_{t} ) \nonumber \\
&= p_{a,t} \label{prob_At1_b} 
\end{align}
From inequalities (\ref{prob_At1_a}) and (\ref{prob_At1_b}),
\begin{equation}
\mathbb{P}(\ArmAt_{t+1} \IsEqual 1 \given E_a^\theta(t), \mathcal{F}_{t}) 
\geq p_{a,t} \cdot \mathbb{P}(M_a(t) \given E_a^\theta(t), \mathcal{F}_{t})  \label{prob_At1}
\end{equation}
Now, let's consider the left hand side of the inequality (\ref{eq:SufStmt}). The fact that  $E_a^\theta(t)$ holds and $\ArmAt_{t+1} \IsEqual a$ implies that $ g_1^{-1} (\theta_1(t)) < g_a^{-1}(\theta_a(t)) < g_a^{-1} (\HThresholdTS_a)$. Hence
\begin{align}
& \mathbb{P}(\ArmAt_{t+1} \IsEqual a \given E_a^\theta(t), \mathcal{F}_{t}) \nonumber \\
&\leq \mathbb{P}\Big(g_1^{-1} (\theta_1(t)) \leq g_a^{-1} (\HThresholdTS_a),  g_a^{-1}(\theta_a(t)) \geq g_j^{-1}(\theta_j(t)), \forall j \neq 1 \given E_a^\theta(t), \mathcal{F}_{t} \Big) \nonumber \\
&= \mathbb{P}\Big(g_1^{-1} (\theta_1(t)) \leq g_a^{-1} (\HThresholdTS_a) \given \mathcal{F}_{t-1} \Big) \cdot \mathbb{P} \Big( g_a^{-1}(\theta_a(t)) \geq g_j^{-1}(\theta_j(t)), \forall j \neq 1 \given E_a^\theta(t), \mathcal{F}_{t} \Big) \nonumber \\
&= (1 - p_{a,t}) \cdot \mathbb{P}(M_a(t) \given E_a^\theta(t), \mathcal{F}_{t}) \label{prob_Ata} 
\end{align}
From inequalities (\ref{prob_At1}) and (\ref{prob_Ata}),
$$ \mathbb{P}(\ArmAt_{t+1} \IsEqual a \given E_a^\theta(t), \mathcal{F}_{t}) \leq \frac{(1 - p_{a,t})}{p_{a,t}} \mathbb{P}(\ArmAt_{t+1} \IsEqual 1\given E_a^\theta(t), \mathcal{F}_{t}) $$ 
\end{proof}

\begin{restatable}{myLemma}{TSLMC} When $g_a$ is increasing (resp. decreasing), for any $x_a' \in \ ]x_a,y_a[$ (resp. $]y_a,x_a[$), when $T$ is large enough, \label{TSCF-lemma3}$$\sum_{t \gets 0 }^{T-1}\bP\Big(\ArmAt_{t+1} \IsEqual a, \overline{E_a^\theta(t)}, E_a^\lambda(t)\Big) \leq \frac{\log(T)}{d(\LThresholdTS_a',\HThresholdTS_a)} + 1.$$ 
\end{restatable}

\begin{proof}
When $g_a$ is increasing, the application of Lemma 3 in \cite{DBLP:conf/aistats/AgrawalG13} directly yields 
\[\sum_{t \gets 0 }^{T-1}\bP\left(\ArmAt_{t+1} \IsEqual a, \overline{E_a^\theta(t)}, E_a^\lambda(t)\right) \leq  \frac{\log{T}}{d(\LThresholdTS_a, \HThresholdTS_a)} + 1.\]
The proof is based on the use of deviation inequalities and a link between the Beta and Binomial c.d.f. that shall also be useful in the decreasing case, that we handle now (using slightly different arguments).
\begin{fact}
 $$ F^{beta}_{\alpha, \beta} (\HThresholdTS) = 1 - F^B_{\alpha + \beta - 1, \HThresholdTS}(\alpha - 1) $$
\end{fact}

Note that for decreasing $g_a$, one has $\overline{E_a^\theta(t)}=\{\theta_a(t) \leq \HThresholdTS_a\}$ and $E_a^\lambda(t)=\{\hat{\lambda}_a(t) > \LThresholdTS_a\}$.
Fix $\LThresholdTS_a'$ such that $\HThresholdTS_a<\LThresholdTS_a'<\LThresholdTS_a$ and let $L'_a(T) = \frac{\log(T)}{d(\LThresholdTS_a',\HThresholdTS_a)}$.
\begin{align*}
&\sum_{t \gets 0 }^{T-1}\bP\left(\ArmAt_{t+1} \IsEqual a, \hat{\lambda}_a(t) > \LThresholdTS_a, \theta_a(t) \leq \HThresholdTS_a\right) \\& \leq  \frac{\log(T)}{d(\LThresholdTS_a',\HThresholdTS_a)} + \sum_{t \gets 0 }^{T-1}\bP\left(\ArmAt_{t+1} \IsEqual a,N_a(t) \leq L'_a(T), \theta_a(t) \leq \HThresholdTS_a, \hat{\lambda}_a(t) > \LThresholdTS_a\right) \\
& \leq \frac{\log(T)}{d(\LThresholdTS_a',\HThresholdTS_a)} + \bE\sum_{t \gets 0 }^{T-1}\sum_{s \gets L'_a(T)}^t\ind_{(\ArmAt_{t+1} \IsEqual a,N_a(t) \IsEqual s , \theta_a(t) \leq \HThresholdTS_a, \hat{\lambda}_{a}(t) > \LThresholdTS_a)} \\
& = \frac{\log(T)}{d(\LThresholdTS_a',\HThresholdTS_a)} + \bE\sum_{t \gets 0 }^{T-1}\sum_{s \gets L'_a(T)}^t\ind_{\left(\ArmAt_{t+1} \IsEqual a,N_a(t) \IsEqual s , \hat{\lambda}_{a}(t) > \LThresholdTS_a\right)}\bP\left(\theta_a(t) \leq \HThresholdTS_a \given \mathcal{F}_{t}\right) \\
& = \frac{\log(T)}{d(\LThresholdTS_a',\HThresholdTS_a)} + \bE\sum_{t \gets 0 }^{T-1}\sum_{s \gets L'_a(T)}^t\ind_{\left(\ArmAt_{t+1} \IsEqual a,N_a(t) \IsEqual s , \hat{\lambda}_{a}(t) > \LThresholdTS_a\right)}F^{beta}_{(s\hat{\lambda}_{a}(t) + 1,s -s\hat{\lambda}_{a}(t)+1)}(\HThresholdTS_a) \\ 
& = \frac{\log(T)}{d(\LThresholdTS_a',\HThresholdTS_a)} + \bE\sum_{t \gets 0 }^{T-1}\sum_{s \gets L'_a(T)}^t\ind_{\left(\ArmAt_{t+1} \IsEqual a,N_a(t) \IsEqual s , \hat{\lambda}_{a}(t) > \LThresholdTS_a\right)}\left(1 - F^{B}_{(s+1,\HThresholdTS_a)}(s\hat{\lambda}_a(t))\right)\\
& \leq \frac{\log(T)}{d(\LThresholdTS_a',\HThresholdTS_a)} + \bE\sum_{t \gets 0 }^{T-1}\sum_{s \gets L'_a(T)}^t\ind_{\left(\ArmAt_{t+1} \IsEqual a,N_a(t) \IsEqual s , \hat{\lambda}_{a}(t) > \LThresholdTS_a\right)}\underbrace{\left(1 - F^{B}_{(s+1,\HThresholdTS_a)}(s\LThresholdTS_a)\right)}_{A_s}
\end{align*}
Introducing $(X_k)$ an i.i.d. sequence drawn from Bernoulli of mean $\HThresholdTS_a$, term $A_s$ can be written, for any $s$, .
\begin{eqnarray*}
A_s &=& \bP\left(\sum_{k \gets 1}^{s+1}X_k \geq \LThresholdTS_as\right) \leq \bP\left(\sum_{k \gets 1}^s X_k \geq \LThresholdTS_a s -1\right) = \bP\left(\frac{1}{s}\sum_{k \gets 1}^s X_k \geq \LThresholdTS_a - \frac{1}{s}\right) \\
& \leq & \exp\left(-s d\left(\LThresholdTS_a - 1/s,\HThresholdTS_a\right)\right) \leq \exp\left(-\log(T) \frac{d(\LThresholdTS_a - 1/s,\HThresholdTS_a)}{d(\LThresholdTS_a',\HThresholdTS_a)}\right) \leq \frac{1}{T}, 
\end{eqnarray*}
for large enough $T$, and $s$ larger than $L'_a(T)$ (as it holds that $d(\LThresholdTS_a -1/s,\HThresholdTS_a) \geq d(\LThresholdTS_a',\HThresholdTS_a)$).
Finally, for $T$ large enough,
\begin{align*}
 &\sum_{t \gets1}^{T}\bP\left(\ArmAt_{t} \IsEqual a, \hat{\lambda}_a(t) \geq \LThresholdTS_a, \theta_a(t) \leq \HThresholdTS_a\right) \\
 &\leq \frac{\log(T)}{d(\LThresholdTS_a',\HThresholdTS_a)} +  \sum_{s \gets 0}^{T-1}\frac{1}{T}\bE\underbrace{\sum_{t \gets s }^{T}\ind_{\left(\ArmAt_{t+1} \IsEqual a,N_a(t) \IsEqual s \right)}}_{\leq 1}  \\
&\leq \frac{\log(T)}{d(\LThresholdTS_a',\HThresholdTS_a)} + \sum_{t \gets1}^{T} \frac{1}{T} = \frac{\log(T)}{d(\LThresholdTS_a',\HThresholdTS_a)} + 1. 
\end{align*} 

\end{proof}


\begin{restatable}{myLemma}{TSLMB}
\label{TSCF-lemma2}
$\sum_{t \gets 0 }^{T-1} \mathbb{P} \Big( \ArmAt_{t+1} \IsEqual a, \overline{E_a^\lambda(t)}\Big) \leq 1 + \frac{1}{d(\LThresholdTS_a, \lambda_a)}.  $
\end{restatable}

\begin{proof}
This result follows from the application of Chernoff bound for the concentration of $\hat{\lambda}_a(t)$. When $g_a$ is increasing, it follows directly from the application of Lemma 2 in \cite{DBLP:conf/aistats/AgrawalG13}, hence we write the proof in the decreasing case only, where we shall justify that for $\LThresholdTS_a < \lambda_a$, 
\[\sum_{t \gets 0 }^{T-1}\bP\left(\ArmAt_{t+1} \IsEqual a , \hat{\lambda}_a(t) < \LThresholdTS_a \right) \leq \frac{1}{d(\LThresholdTS_a,\lambda_a)} +1.\]
Using $\hat{\lambda}_{a,s}$ to denote the empirical mean of the $s$ first observations from the feedback of arm $a$,
\begin{eqnarray*}
\sum_{t \gets 0 }^{T-1}\bP\left(\ArmAt_{t+1} \IsEqual a , \hat{\lambda}_a(t) < \LThresholdTS_a \right) &=&  \bE\left[\sum_{t \gets 0 }^{T-1} \sum_{s \gets 0}^t \ind_{(\ArmAt_{t+1} \IsEqual a,N_a(t) \IsEqual s)} \ind_{(\hat{\lambda}_{a,s}<\LThresholdTS_a)}\right] \\
& = & \bE\Big[\sum_{s \gets 0}^T\ind_{\left(\hat{\lambda}_{a,s}<\LThresholdTS_a\right)} \underbrace{\sum_{t \gets s}^{T} \ind_{(\ArmAt_{t+1} \IsEqual a,N_a(t) \IsEqual s)}}_{\leq 1} \Big] \\
& \leq & 1 + \sum_{s \gets 1}^{T-1}\bP\left(\hat{\lambda}_{a,s} < \LThresholdTS_a\right)  \leq  1 + \sum_{s \gets 1}^{T-1}\exp(-sd(\LThresholdTS_a,\lambda_a)) \\
& \leq & 1 + \frac{1}{d(\LThresholdTS_a,\lambda_a)},
\end{eqnarray*}
where the last but one inequality follows from Chernoff inequality (as $\LThresholdTS_a < \lambda_a$).
\end{proof}

\begin{restatable}{myLemma}{TSLMD} Let $\tau_s$ be the instant of the $s$-th play of arm 1. Then there exists a function $f(s)=f(s,\lambda_1,g_1(g_a^{-1}(\mu_1)))$ satisfying 
$\sum_{s=1}^\infty f(s) < \infty$ such that for all $s$, 
\label{TSCF-lemma4}
$$ \bE \left[ \frac{1}{p_{a, \tau_s + 1}} \right] \leq 1+f(s). $$
\end{restatable}
\begin{proof}
Let $\tilde{\HThresholdTS}_a \defined g_1(g^{-1}_a (\HThresholdTS_a))$. Examining all possibilities, one can easily show that
\begin{itemize}
 \item if $g_1$ is increasing and $g_a$ is increasing, $p_{a,t} = \bP\left(\theta_1(t) > \tilde{\HThresholdTS}_a\right)$, with $\tilde{\HThresholdTS}_a < \lambda_1$,
 \item if $g_1$ is increasing and $g_a$ decreasing, $p_{a,t} = \bP\left(\theta_1(t) > \tilde{\HThresholdTS}_a\right)$, with $\tilde{\HThresholdTS}_a < \lambda_1$,
 \item if $g_1$ is decreasing and $g_a$ is increasing, $p_{a,t} = \bP\left(\theta_1(t) < \tilde{\HThresholdTS}_a\right)$, with $\tilde{\HThresholdTS}_a > \lambda_1$,
 \item if $g_1$ is decreasing and $g_a$ is decreasing, $p_{a,t} = \bP\left(\theta_1(t) < \tilde{\HThresholdTS}_a\right)$, with $\tilde{\HThresholdTS}_a > \lambda_1$.
 \end{itemize}

When $g_1$ is increasing, $\tilde{\HThresholdTS}_a < \lambda_1$ and  
\[p_{a,\tau_s + 1} = 1-F^{beta}_{(S_1(\tau_s)+1 , s - S_1(\tau_s)+1)}(\tilde{\HThresholdTS}_a) = F^B_{(s+1,\tilde{\HThresholdTS}_a)}(S_1(\tau_s)).\]
Using that $S_1(\tau_s)$ has a binomial distribution with parameters $(s,\lambda_1)$ yields 
\begin{eqnarray}
\bE\left[\frac{1}{p_{a,\tau_s +1}}\right]   =  \sum_{j \gets 0}^s \frac{f^B_{(s,\lambda_1)}(j)}{F^B_{(s+1,\tilde{\HThresholdTS}_a)}(j)}. \label{ToBound1}  
\end{eqnarray}

When $g_1$ is decreasing, recall $\tilde{\HThresholdTS}_a > \lambda_1$ and one has 
\[p_{a,\tau_s + 1} = F^{beta}_{(S_1(\tau_s)+1 , s - S_1(\tau_s)+1)}(\tilde{\HThresholdTS}_a) = 1 - F^B_{(s+1,\tilde{\HThresholdTS}_a)}(S_1(\tau_s)).\]
Using again the distribution of $S_1(\tau_s)$ yields 
\begin{eqnarray*}
\bE\left[\frac{1}{p_{a,\tau_s +1}}\right]   =  \sum_{j \gets 0}^s \frac{f^B_{(s,\lambda_1)}(j)}{1- F^B_{(s+1,\tilde{\HThresholdTS}_a)}(j)} 
\end{eqnarray*}
Note here two simple properties of Binomial distributions: for all $t\in \mathbb{N}^*$ and $c\in [0,1]$, for all $j \in \{0,\dots,t\}$,   
\begin{itemize}
 \item $f^B_{(t,c)}(j) = f_{(t,1-c)}(s-j)$ 
 \item $F^B_{(t,c)}(j) = 1 - F^B_{(t,1-c)}(t-j-1)$
\end{itemize}
It follows that 
\begin{eqnarray}
\bE\left[\frac{1}{p_{a,\tau_s +1}}\right]   =  \sum_{j \gets 0}^s \frac{f^B_{(s,1-\lambda_1)}(s-j)}{F^B_{(s+1,1-\tilde{\HThresholdTS}_a)}(s-j)}=  \sum_{j \gets 0}^s \frac{f^B_{(s,1-\lambda_1)}(j)}{F^B_{(s+1,1-\tilde{\HThresholdTS}_a)}(j)},\label{ToBound2}
\end{eqnarray}
with $1-\lambda_1 > 1-\tilde{\HThresholdTS}_a$. 

The proof for Lemma 4 given in \cite{DBLP:conf/aistats/AgrawalG13} provides an upper bound on the quantity
\[  \sum_{j \gets 0}^s \frac{f^B_{(s,c)}(j)}{F^B_{(s+1,c)}(j)}\]
whenever $c$ is larger that $d$. Using this result one can bound \eqref{ToBound1} and \eqref{ToBound2} by the same quantity: 
$$ 
\bE \left[ \frac{1}{p_{a, \tau_s + 1}} \right] \leq  
\begin{cases}
1 + \frac{3}{\Delta'_a}, & \text{if } s < \frac{8}{\Delta'_a} \\
   1 + \Theta\left(\exp{(-{\Delta'_a}^2 s/2)} + \frac{1}{(s+1){\Delta'_a}^2}\exp{(-D_as)} + \frac{1}{\exp{({\Delta'_a}^2s/4)}-1}\right) , & \text{if } s \geq \frac{8}{\Delta'_a}   
\end{cases}
$$
where $\Delta'_a \defined \lambda_1 - \tilde{\HThresholdTS}_a$ and $D_a \defined \tilde{\HThresholdTS}_a \log{\frac{\tilde{\HThresholdTS}_a}{\lambda_1}} + (1 - \tilde{\HThresholdTS}_a) \log{\frac{1-\tilde{\HThresholdTS}_a}{1-\lambda_1}}$. Hence, Lemma~\ref{TSCF-lemma4} follows with 
\[f(s) \defined \begin{cases}
 \frac{3}{\Delta'_a}, & \text{if } s < \frac{8}{\Delta'_a} \\
   \Theta\left(\exp{(-{\Delta'_a}^2 s/2)} + \frac{1}{(s+1){\Delta'_a}^2}\exp{(-D_as)} + \frac{1}{\exp{({\Delta'_a}^2s/4)}-1}\right) , & \text{if } s \geq \frac{8}{\Delta'_a}   
\end{cases},\]
that satisfies $\sum_{s \gets 0}^\infty f(s) < \infty$. 
\end{proof}

One can now complete the proof of Theorem~\ref{UB_TSCF_Theorem}.

\begin{align*}
\bE[N_a(T)] &=\sum_{t \gets 0 }^{T-1} \bP(\ArmAt_{t+1} \IsEqual a) \\
&= \sum_{t \gets 0 }^{T-1} \bP(\ArmAt_{t+1} \IsEqual a, E_a^\lambda(t), E_a^\theta(t)) + \sum_{t \gets 0 }^{T-1} \bP(\ArmAt_{t+1} \IsEqual a, E_a^\lambda(t), \overline{E_a^\theta(t)}) \\
&+ \sum_{t \gets 0 }^{T-1} \bP(\ArmAt_{t+1} \IsEqual a, \overline{E_a^\lambda(t)}) \\
&\leq \sum_{t \gets 0 }^{T-1} \bE \left[ \frac{(1 - p_{a,t})}{p_{a,t}} \ind_{(\ArmAt_{t+1} \IsEqual 1, E_a^\theta(t), E_a^\lambda(t))}\right] + \frac{\log{T}}{d(\LThresholdTS'_a, \HThresholdTS_a)} + 1 + \frac{1}{d(\LThresholdTS_a, \lambda_a)} + 1\\
&\leq \sum_{s \gets 0}^{T-1} \bE \left[ \frac{(1 - p_{a, \tau_s + 1})}{p_{a, \tau_s + 1}} \sum_{t \gets\tau_s}^{\tau_{s + 1}-1} \ind_{(\ArmAt_{t+1} \IsEqual 1)}\right] + \frac{\log{T}}{d(\LThresholdTS'_a, \HThresholdTS_a)} + 1 + \frac{1}{d(\LThresholdTS_a, \lambda_a)} + 1\\
&= \sum_{s \gets 0}^{T-1} \bE \left[ \frac{1}{p_{a, \tau_s + 1}} - 1 \right] + \frac{\log{T}}{d(\LThresholdTS'_a, \HThresholdTS_a)} + 1 + \frac{1}{d(\LThresholdTS_a, \lambda_a)} + 1\\ 
&\leq \frac{\log{T}}{d(\LThresholdTS'_a, \HThresholdTS_a)} +  \sum_{s \gets 0}^{T-1}f(s) + \frac{1}{d(\LThresholdTS_a, \lambda_a)} + 2. 
\end{align*}

Fix $\psi >0$. Using the monotonicity properties of the divergence function $d$, there exists $\LThresholdTS_a < \LThresholdTS_a' < \HThresholdTS_a$ in the increasing case and  $\LThresholdTS_a>\LThresholdTS_a'>\HThresholdTS_a$ in the decreasing case such that $d(\LThresholdTS_a',\HThresholdTS_a) \geq d(\lambda_a,g_a(\mu_1))/(1+\psi)$. For this particular choice, one obtains 
\[\bE[N_a(T)] \leq (1+\psi)\frac{\log(T)}{d(\lambda_a,g_a(\mu_a))} + R(\LThresholdTS_a,\LThresholdTS_a',\HThresholdTS_a),\]
where $R(\LThresholdTS_a,\LThresholdTS_a',\HThresholdTS_a)$ is a rest term that depends on $\psi, \mu_1, \mu_a, g_1$ and $g_a$. The result follows using that $\cRegret_T = \sum_{a \gets2}^K\Delta_a\bE[N_a(T)]$.
\end{proof}

\section{Additional Empirical Evaluation}
\label{sec:eval2}
We ran the experiments mentioned in Section \ref{sec:expt1}, \ref{sec:expt2} and \ref{sec:expt3} on 4 additional Bernoulli corrupt bandit problems. These problems are succinctly described by the mean rewards of their arms given in Table \ref{tab2}. Recall that in the experiment to compare the performance of the algorithms over a period of time, randomized response was employed to corrupt the feedback and according to Equation~(\ref{eq:CorruptionMatrix}), $p_{00} = p_{11} = 0.6$ for the optimal arm, while for all the other arms, both $p_{00}$ and $p_{11}$ were set to $0.9$. The time horizon was varied to $10^5$ and each experiment was repeated 1000 times. Figures \ref{Fig2A}, \ref{Fig3A}, \ref{Fig4A} and \ref{Fig5A} show the average regret of the considered algorithms. In the second experiment aiming to see the effect of various levels of differential privacy on the regret, we chose $\epsilon$ from the set $\{1/8, 1/4, 1/2, 1, 2, 4, 8\}$. The corruption parameters are set by substituting the values of $\epsilon$ in Equation~(\ref{eq:DPmatrix}). The horizon was fixed to $10^5$ and the experiment was repeated $1000$ times. The corresponding curve for the average regret are given in Figures \ref{Fig2B}, \ref{Fig3B}, \ref{Fig4B} and \ref{Fig5B}. The third experiment compares the regret of \KLUCBCF \ and \TSCF \ with \textsc{DP-UCB-Int} for $\epsilon = 1$ and its results are given in Figures \ref{Fig2C}, \ref{Fig3C}, \ref{Fig4C} and \ref{Fig5C} .
\begin{table}[!htbp]
\centering
\caption{Bernoulli mean arm rewards for experimental scenarios}
\begin{tabular}{cccccccccccc}
  \addlinespace
    \toprule
     \multirow{2}*{Scenario} & & & & & Arms \\
     &1 & 2 & 3 & 4 & 5 & 6 & 7 & 8 & 9 & 10\\
     \midrule 
      1 & 0.9 & 0.6 \\ 
    \midrule    
      2 & 0.9 & 0.8 \\ 
    \midrule
     3 & 0.9 & 0.8 & 0.8 & 0.8 & 0.7 & 0.7 & 0.7 & 0.6 & 0.6 & 0.6  \\  
	\midrule
     4 & 0.9 & 0.6 & 0.6 & 0.6 & 0.6 & 0.6 & 0.6 & 0.6 & 0.6 & 0.6 \\ 
    \bottomrule
  \end{tabular} 
  \label{tab2}
\end{table}

\begin{figure}[!htbp]
\centering
\begin{subfigure}{.3\textwidth}
  \centering
  \includegraphics[width=\linewidth]{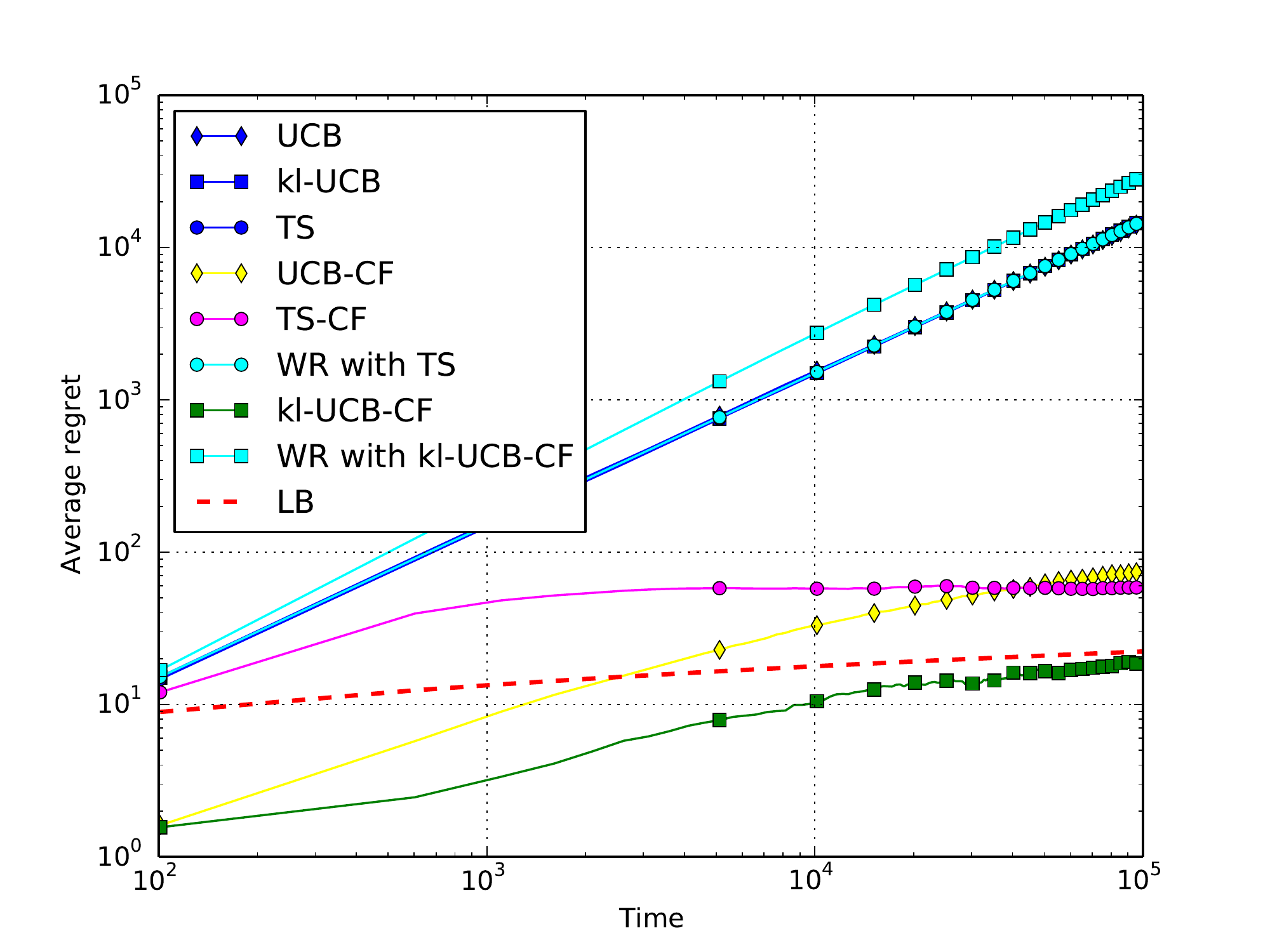}
  \caption{Regret plots for \KLUCBCF \ and \TSCF \ with others for varying  horizons up to $10^5$}
  \label{Fig2A}
\end{subfigure}%
\hspace{1em}
\begin{subfigure}{.3\textwidth}
  \centering
  \includegraphics[width=\linewidth]{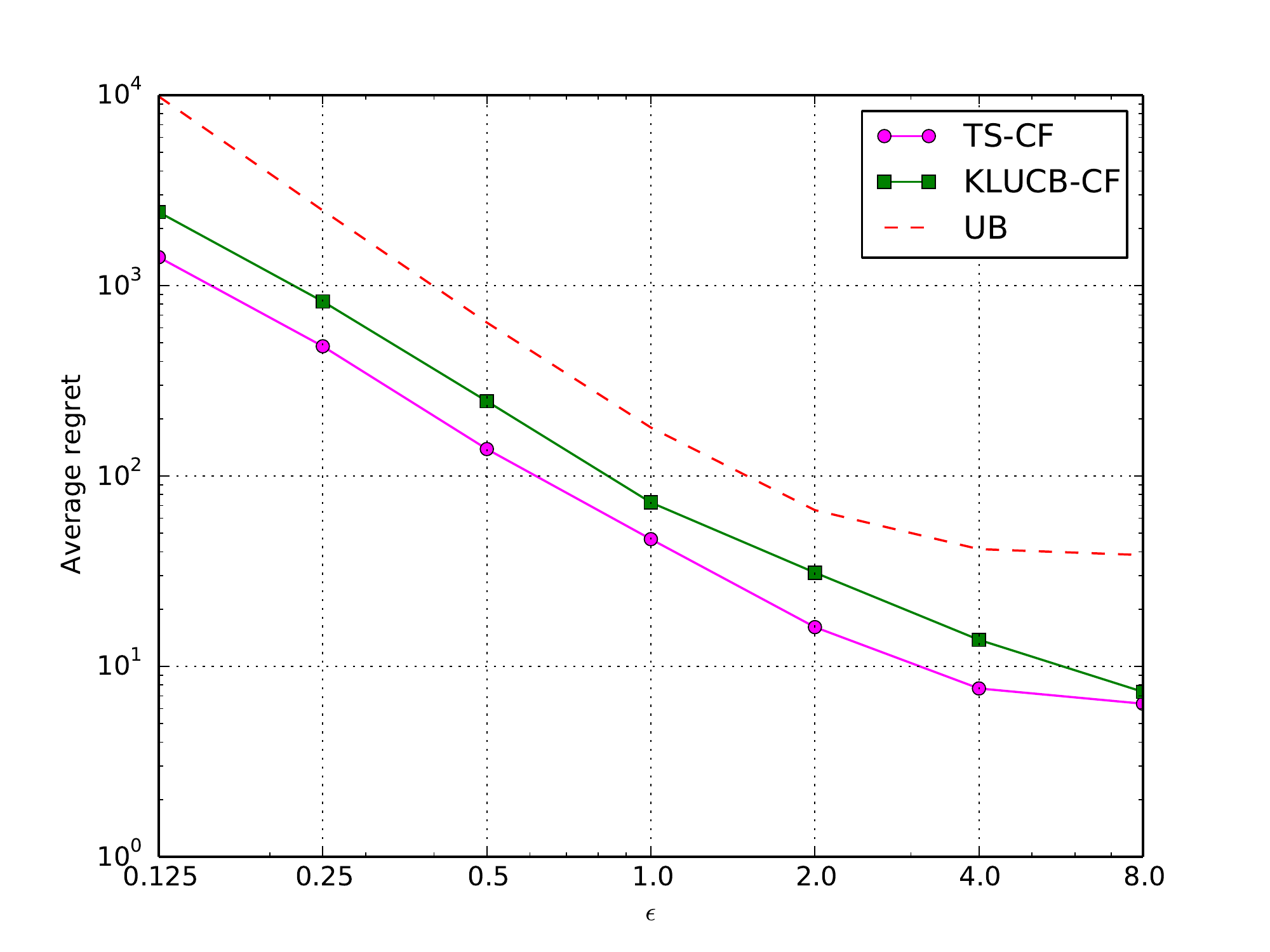}
  \caption{Regret plots for \KLUCBCF \ and \TSCF \ with $\epsilon$ = $\{\frac{1}{8}, \frac{1}{4}, \frac{1}{2}, 1, 2, 4, 8\}$, $T=10^5$}
  \label{Fig2B} 
\end{subfigure}
\hspace{1em}
\begin{subfigure}{.3\textwidth}
  \centering
  \includegraphics[width=\linewidth]{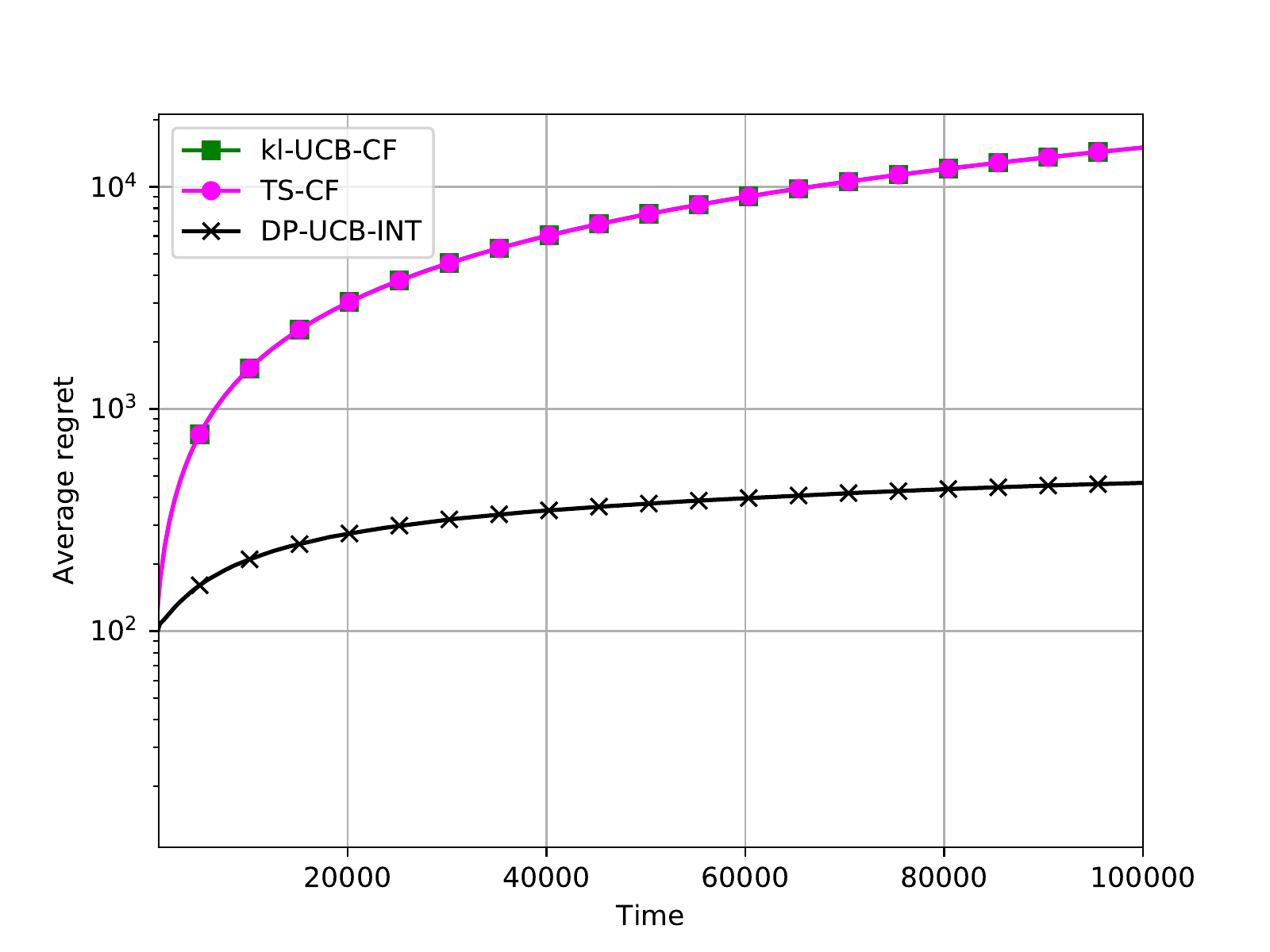}
  \caption{Regret plots for $\epsilon = 1$ for  \KLUCBCF \ and \TSCF \,  \textsc{DP-UCB-Int}}
  \label{Fig2C}
\end{subfigure}
\caption{Regret plots for scenario 1}
\label{Fig2}
\end{figure}

\begin{figure}[!htbp]
\centering
\begin{subfigure}{.3\textwidth}
  \centering
  \includegraphics[width=\linewidth]{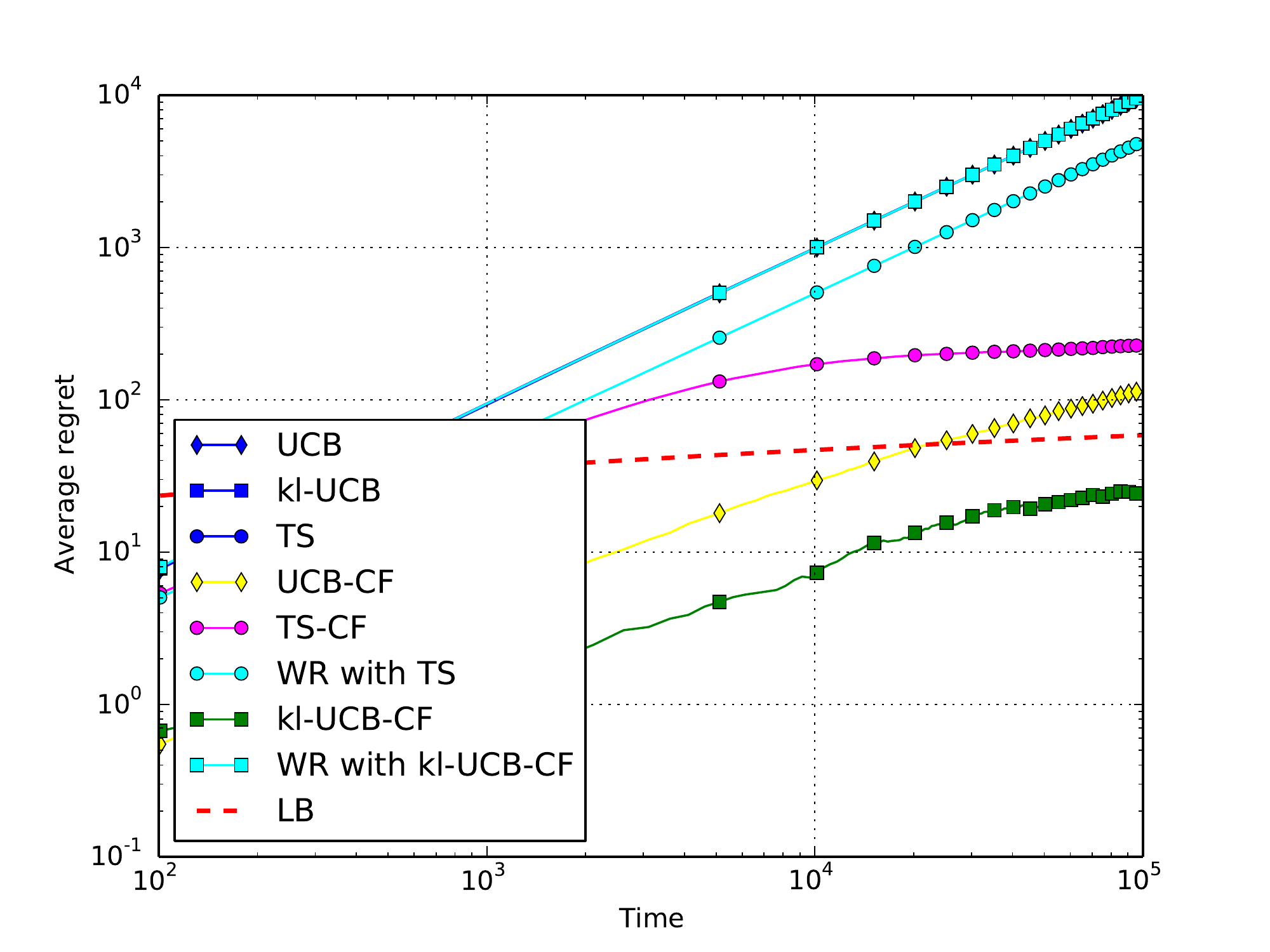}
  \caption{Regret plots for \KLUCBCF \ and \TSCF \ with others for varying  horizons up to $10^5$}
  \label{Fig3A}
\end{subfigure}%
\hspace{1em}
\begin{subfigure}{.3\textwidth}
  \centering
  \includegraphics[width=\linewidth]{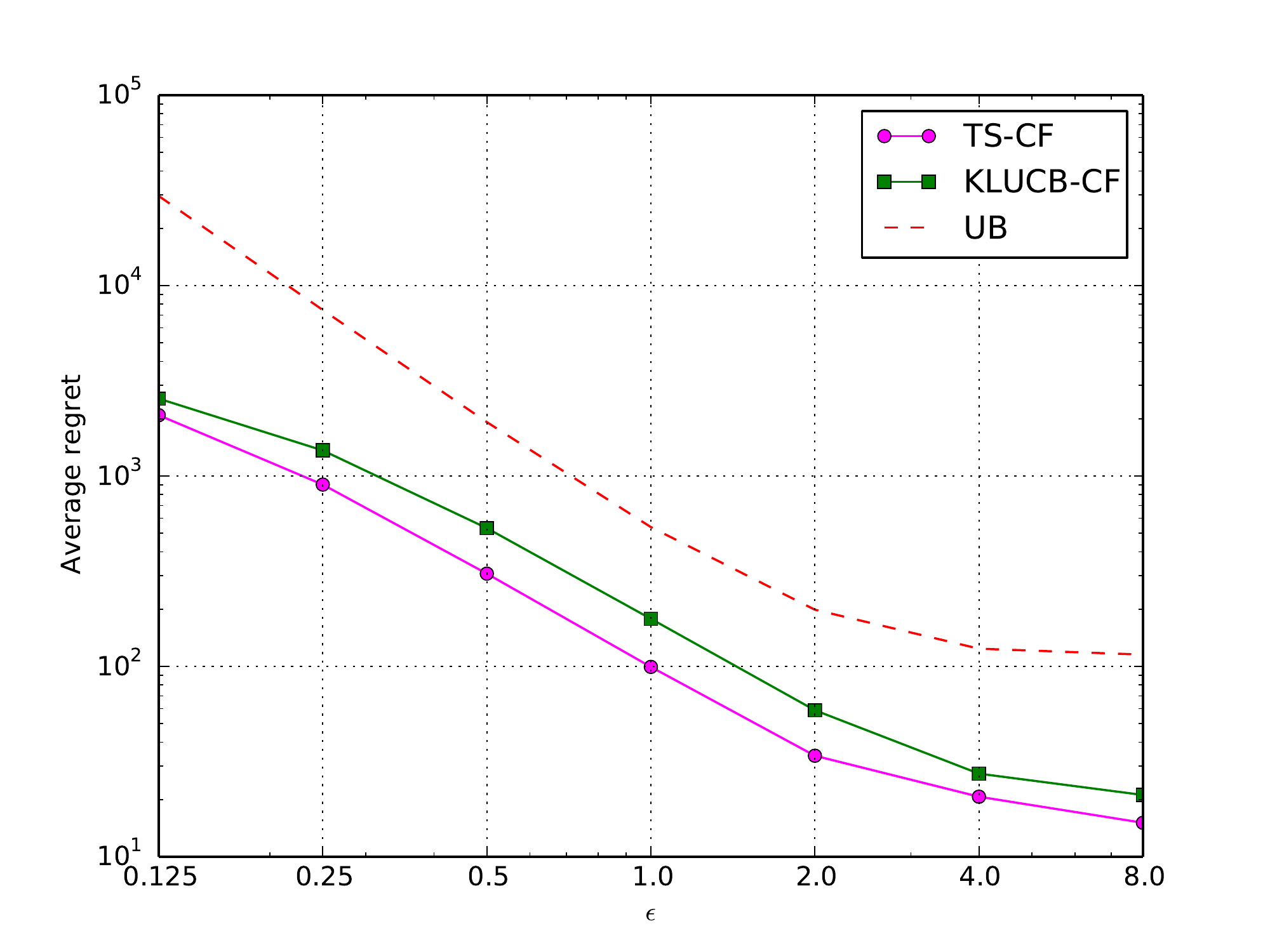}
  \caption{Regret plots for \KLUCBCF \ and \TSCF \ with $\epsilon$ = $\{\frac{1}{8}, \frac{1}{4}, \frac{1}{2}, 1, 2, 4, 8\}$, $T=10^5$}
  \label{Fig3B}
\end{subfigure}
\hspace{1em}
\begin{subfigure}{.3\textwidth}
  \centering
  \includegraphics[width=\linewidth]{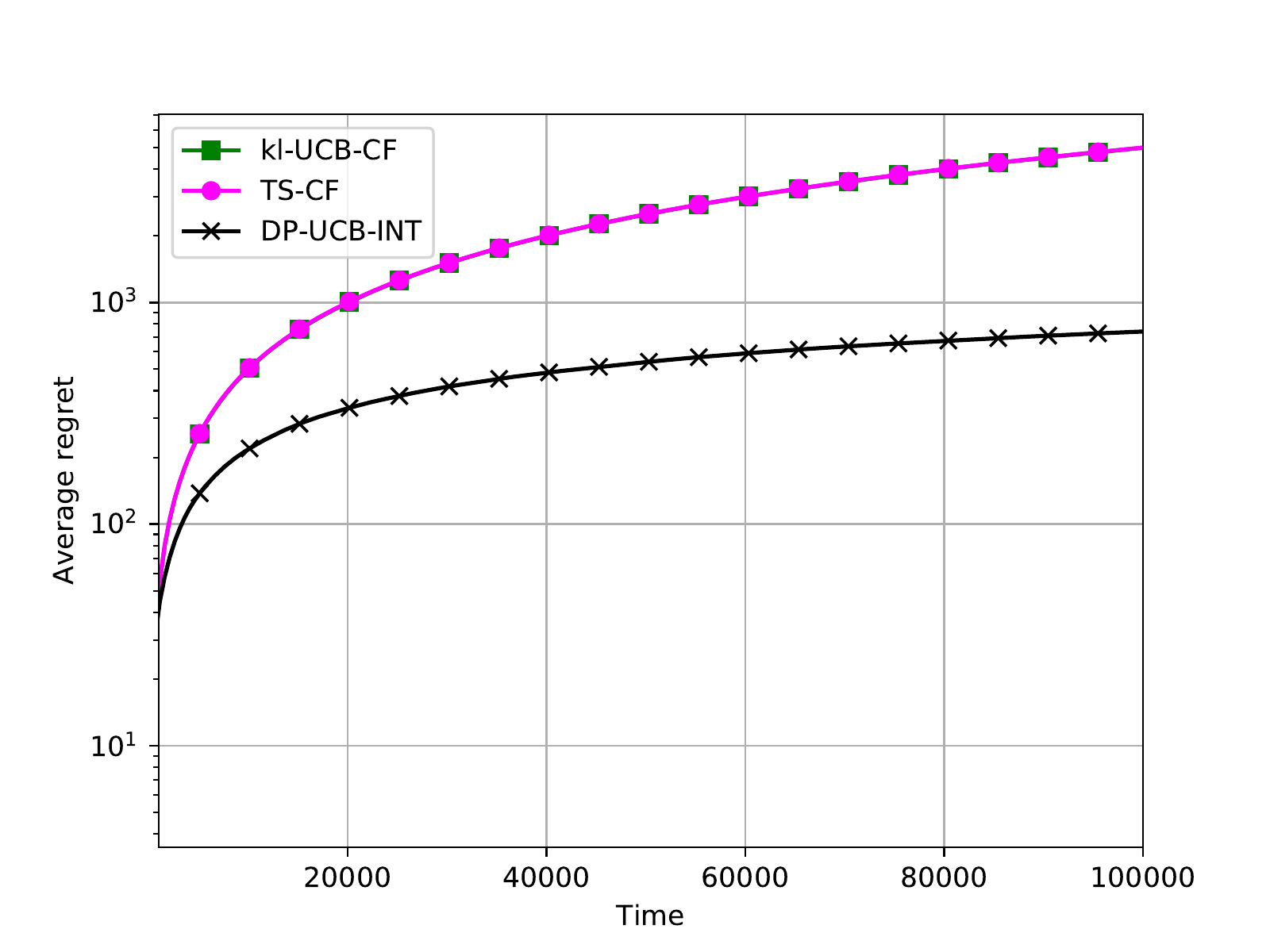}
  \caption{Regret plots for $\epsilon = 1$ for  \KLUCBCF \ and \TSCF \,  \textsc{DP-UCB-Int}}
  \label{Fig3C}
\end{subfigure}
\caption{Regret plots for scenario 2}
\label{Fig3}
\end{figure}

\begin{figure}[!htbp]
\centering
\begin{subfigure}{.3\textwidth}
  \centering
  \includegraphics[width=\linewidth]{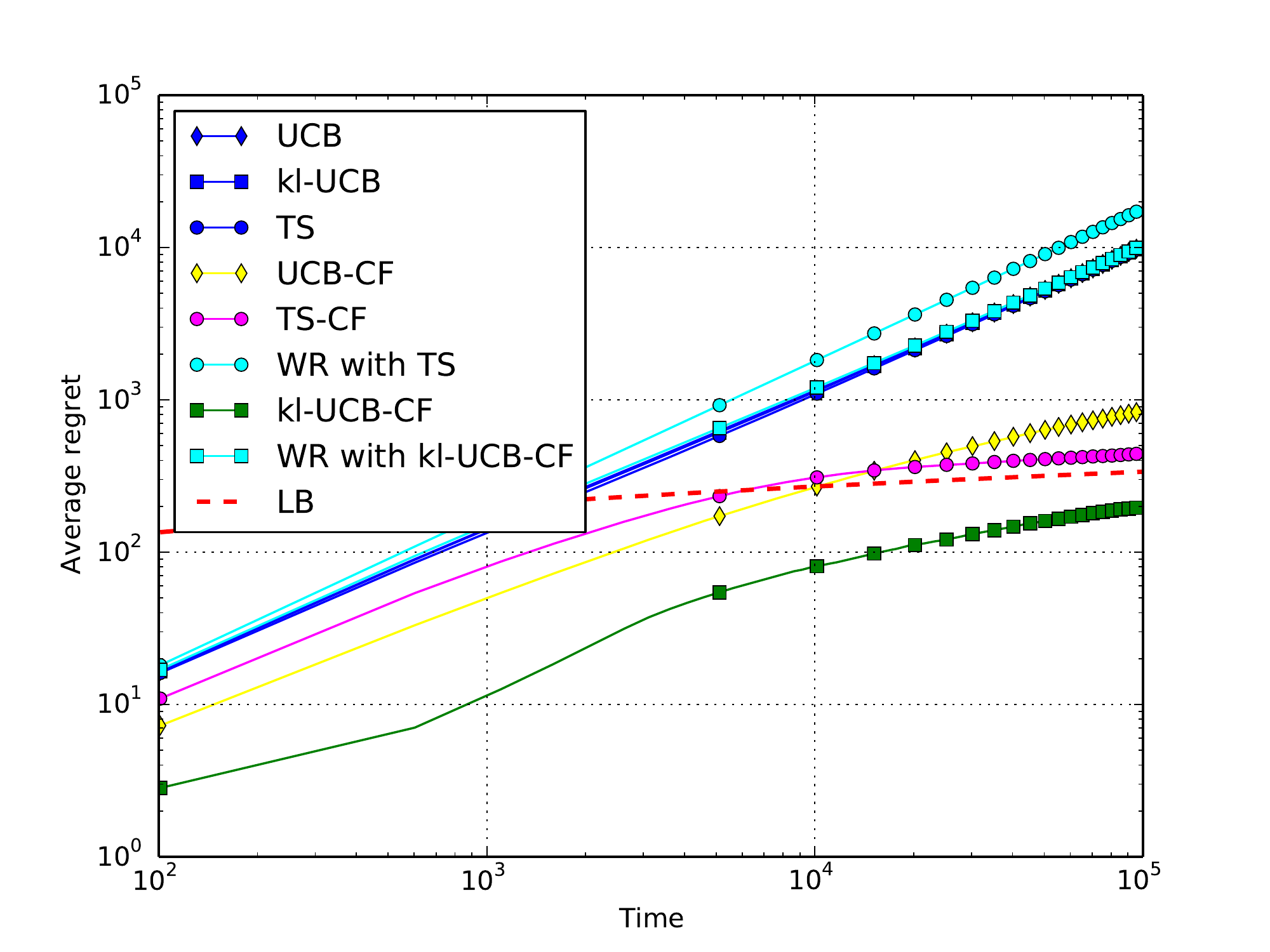}
  \caption{Regret plots for \KLUCBCF \ and \TSCF \ with others for varying  horizons up to $10^5$}
  \label{Fig4A}
\end{subfigure}%
\hspace{1em}
\begin{subfigure}{.3\textwidth}
  \centering
  \includegraphics[width=\linewidth]{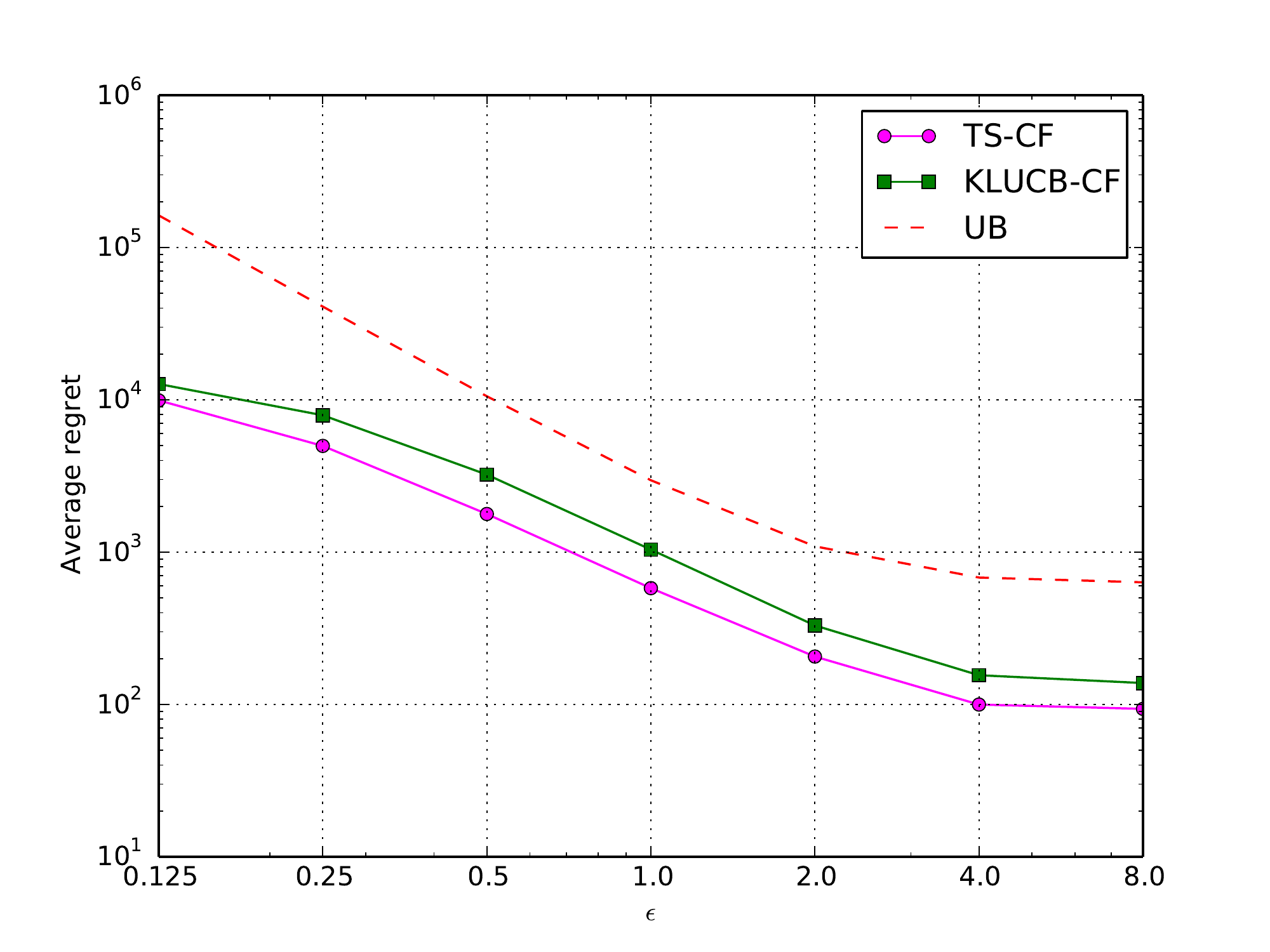}
  \caption{Regret plots for \KLUCBCF \ and \TSCF \ with $\epsilon$ = $\{\frac{1}{8}, \frac{1}{4}, \frac{1}{2}, 1, 2, 4, 8\}$, $T=10^5$}
  \label{Fig4B}
\end{subfigure}
\hspace{1em}
\begin{subfigure}{.3\textwidth}
  \centering
  \includegraphics[width=\linewidth]{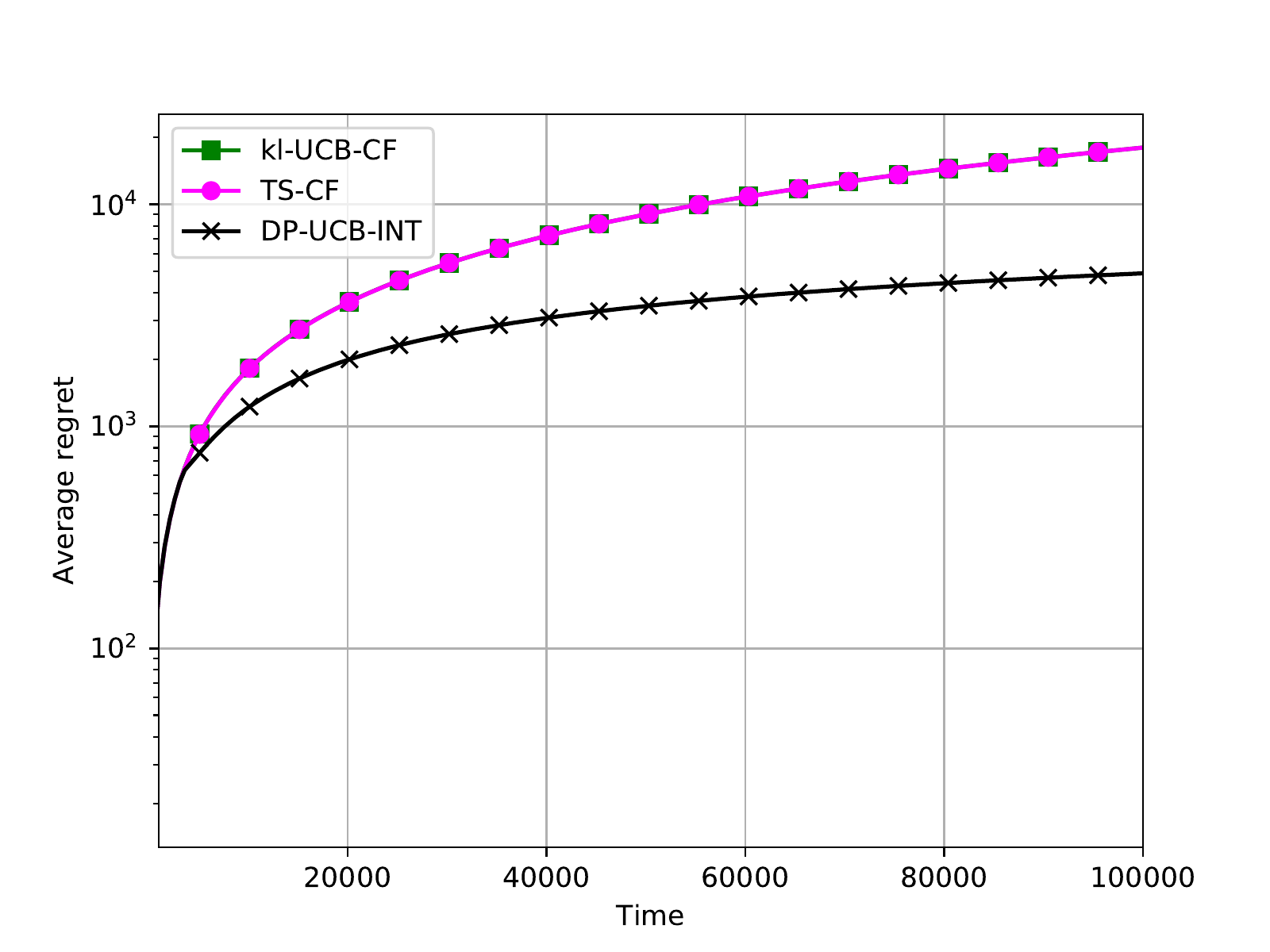}
  \caption{Regret plots for $\epsilon = 1$ for  \KLUCBCF \ and \TSCF \,  \textsc{DP-UCB-Int}}
  \label{Fig4C}
\end{subfigure}
\caption{Regret plots for scenario 3}
\label{Fig4}
\end{figure}
\FloatBarrier

\begin{figure}[!htbp]
\centering
\begin{subfigure}{.3\textwidth}
  \centering
  \includegraphics[width=\linewidth]{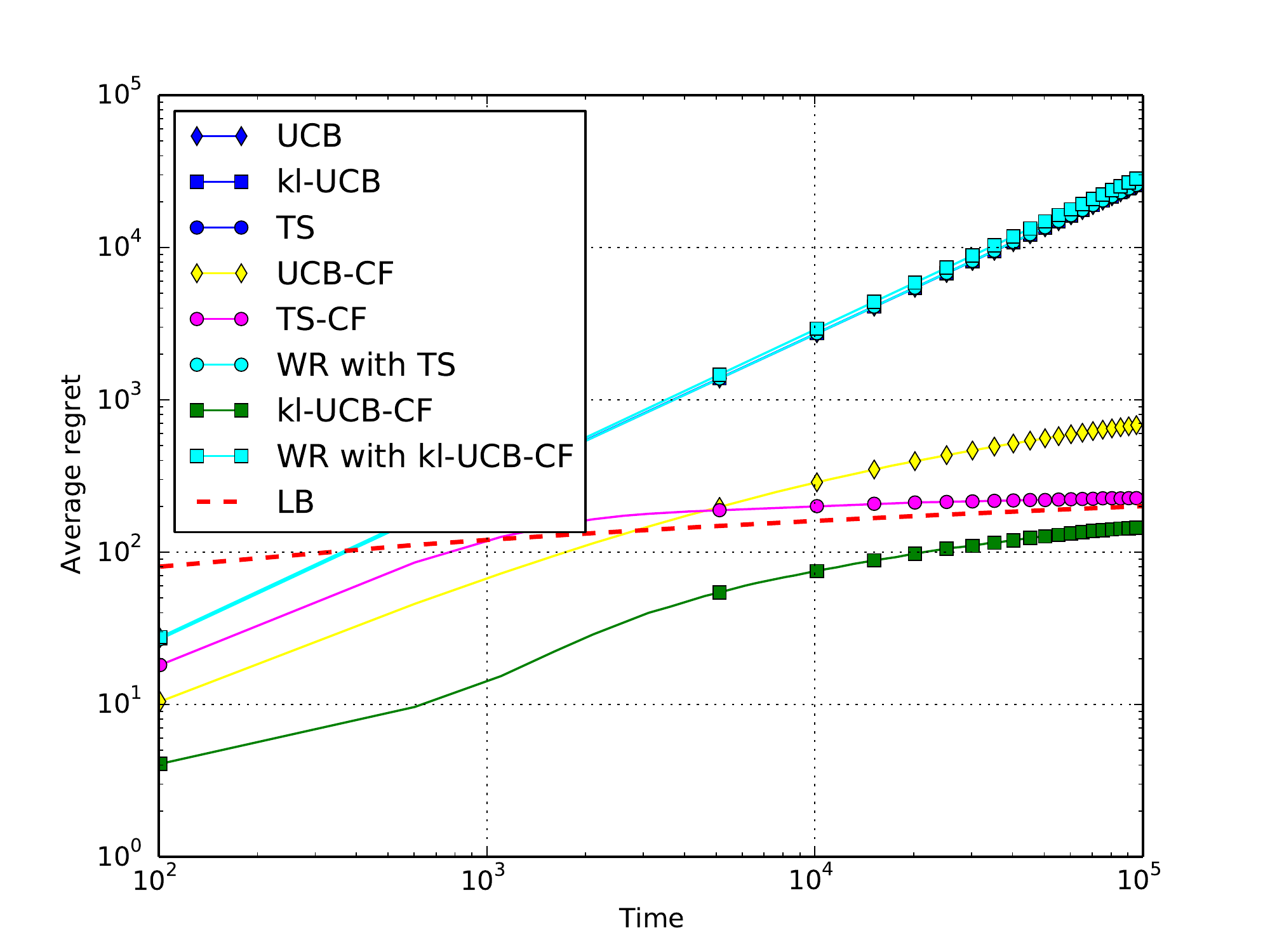}
  \caption{Regret plots for \KLUCBCF \ and \TSCF \ with others for varying  horizons up to $10^5$}
  \label{Fig5A}
\end{subfigure}%
\hspace{1em}
\begin{subfigure}{.3\textwidth}
  \centering
  \includegraphics[width=\linewidth]{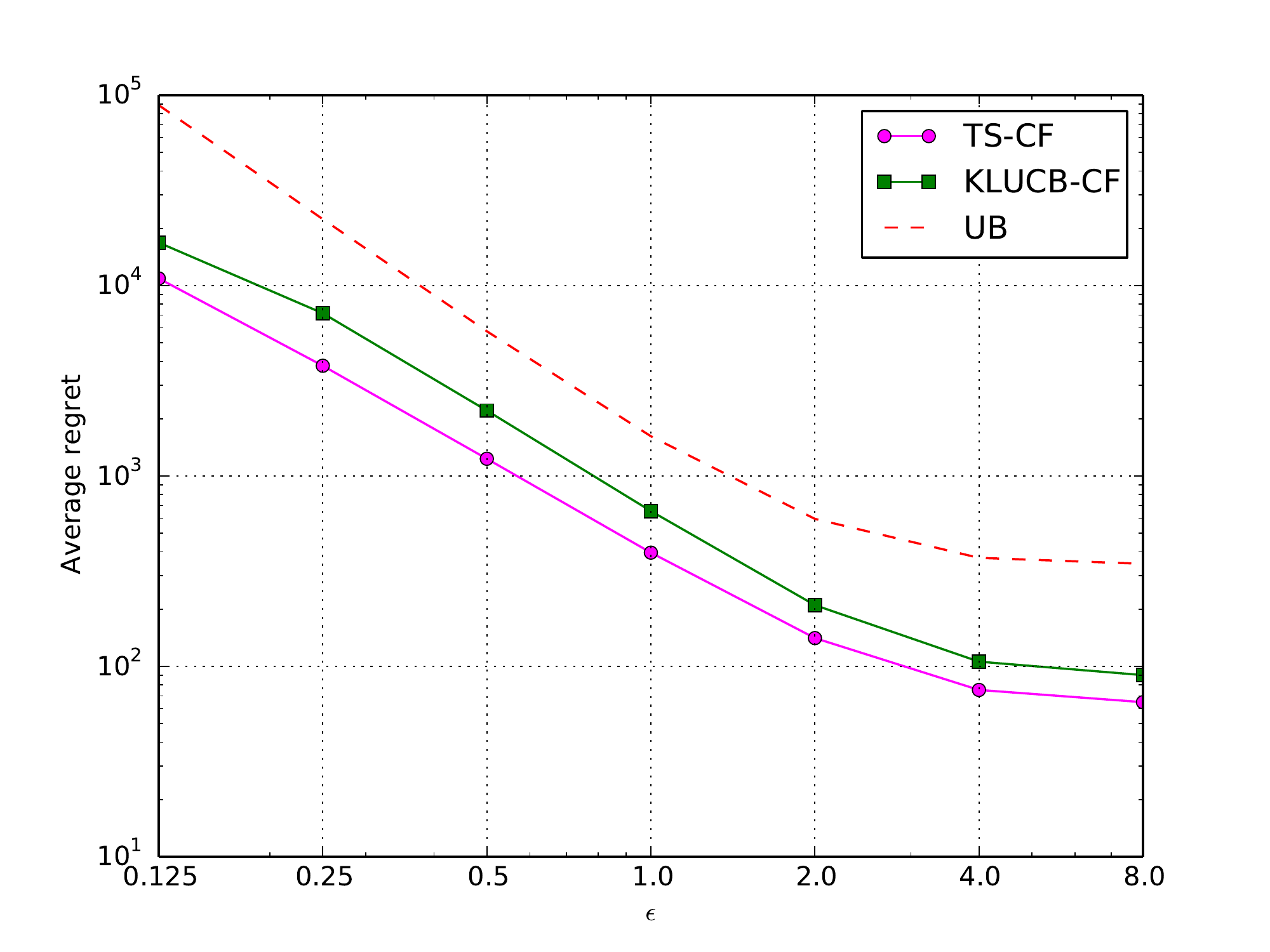}
  \caption{Regret plots for \KLUCBCF \ and \TSCF \ with $\epsilon$ = $\{\frac{1}{8}, \frac{1}{4}, \frac{1}{2}, 1, 2, 4, 8\}$, $T=10^5$}
  \label{Fig5B}
\end{subfigure}
\hspace{1em}
\begin{subfigure}{.3\textwidth}
  \centering
  \includegraphics[width=\linewidth]{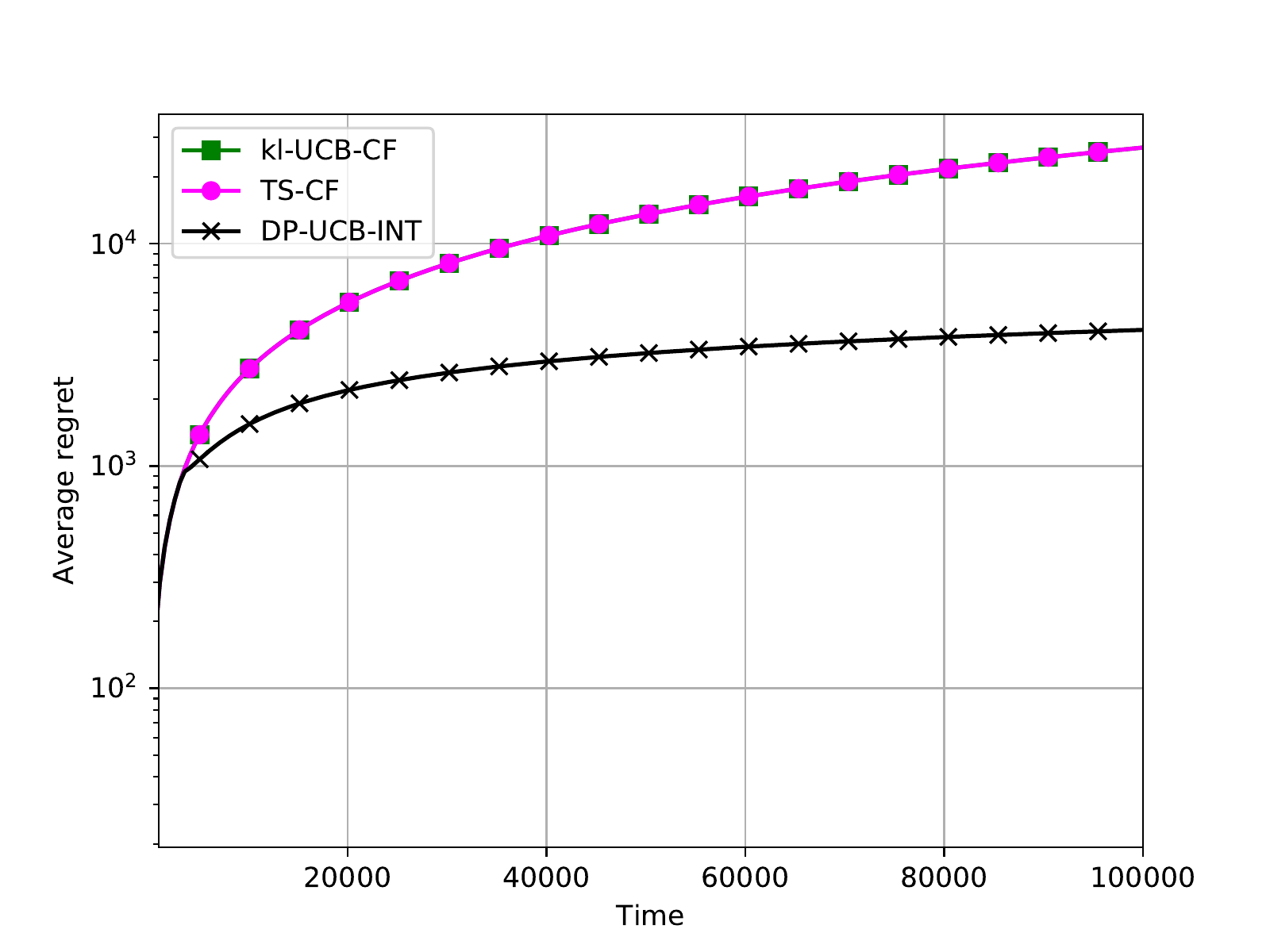}
  \caption{Regret plots for $\epsilon = 1$ for  \KLUCBCF \ and \TSCF \,  \textsc{DP-UCB-Int}}
  \label{Fig5C}
\end{subfigure}
\caption{Regret plots for scenario 4}
\label{Fig5}
\end{figure}

\end{document}